\documentclass[12pt]{article}

\usepackage{ifthen,graphicx,multirow,mathrsfs,xcolor,textcomp,manyfoot,booktabs,algorithm,algorithmicx,algpseudocode,listings,bm}
\usepackage{amsmath,amssymb,amsfonts,amsthm}
\usepackage[title]{appendix}
\usepackage{url}
\usepackage{fullpage}

\allowdisplaybreaks

\newcommand{\be}{\begin{equation}}
\newcommand{\ee}{\end{equation}}
\DeclareMathOperator*{\argmax}{arg\,max}
\DeclareMathOperator*{\argmin}{arg\,min}

\newcommand{\pr}[2][]{\mathrm{Pr}\ifthenelse{\not\equal{}{#1}}{_{#1}}{}\!\left[#2\right]}

\newcommand{\todonote}[2][]{}
\newcommand{\questionnote}[2][]{}
\newcommand{\todonotedone}[2][]{}
\newcommand{\todonoteinline}[2][]{}
\newcommand{\marginnote}[1]{}
\newcommand{\ignore}[1]{}
\newcommand{\gmargin}[1]{}

\theoremstyle{thmstyleone}
\newtheorem{theorem}{Theorem}

\newtheorem{remark}[theorem]{Remark}

\newtheorem{lemma}[theorem]{Lemma}
\newtheorem{claim}[theorem]{Claim}
\newtheorem{corollary}[theorem]{Corollary}

\numberwithin{theorem}{section} 
\theoremstyle{thmstyletwo}

\theoremstyle{thmstylethree}

\begin{document}
\title{Mean and Variance Estimation Complexity in Arbitrary Distributions via Wasserstein Minimization\thanks{Supported in part by a Discovery Grant from the Natural Sciences and Engineering Research Council (NSERC) of Canada. %Notwithstanding any publisher policy or contract, the authors reserve the right to post the final accepted version of this manuscript on arxiv.org in order to comply with NSERC's open access policy.
}}

\author{Valentio Iverson\thanks{Department of Combinatorics \& Optimization,
    University of Waterloo, Waterloo, Ontario, N2L 3G1, Canada, viverson@uwaterloo.ca} \and
  Stephen A. Vavasis\thanks{Department of Combinatorics \& Optimization,
    University of Waterloo, Waterloo, Ontario, N2L 3G1, Canada, vavasis@uwaterloo.ca}}

\maketitle 

\abstract{Parameter estimation is a fundamental challenge in machine learning, crucial for tasks such as neural network weight fitting and Bayesian inference. This paper focuses on the complexity of estimating the globally optimal translation $\bm{\mu} \in \mathbb{R}^l$ and shrinkage $\sigma \in \mathbb{R}_{++}$ parameters for a distribution of the form $\frac{1}{\sigma^l} f_0 \left( \frac{\bm{x} - \bm{\mu}}{\sigma} \right)$, where $f_0$ is a known density in $\mathbb{R}^l$ given $n$ samples. %{\color{red}
Our emphasis is complexity-theoretic rather than asymptotic statistical optimality; nevertheless, Wasserstein objectives are also naturally robust to small perturbations in sample locations, since nearby empirical point masses remain close in transport distance. %} 
We highlight that while the problem is NP-hard for Maximum Likelihood Estimation (MLE), it is possible to obtain $\varepsilon$-approximations for arbitrary $\varepsilon > 0$ within $\text{poly} \left( \frac{1}{\varepsilon} \right)$ time using the Wasserstein distance.}


\section{Introduction}
A fundamental challenge in machine learning is the estimation of distribution parameters. For instance, in neural networks, estimating the parameters of the model to fit into distribution that underlies the weights can significantly impact model performance. Accurate parameter estimation is also crucial in other area like Bayesian inference, where it affects the choice of prior and posterior distributions.

Recently, modern machine learning methods have increasingly leveraged optimal transport-based distances due to their ability to robustly compare probability measures \cite{f13cd90fad03be74cb538ff1f3ea094fa07a4bda} \cite{cao2019multimarginalwassersteingan} \cite{imfeld2024transformerfusionoptimaltransport}. Optimal transport has emerged as a versatile framework for addressing a range of machine learning problems, such as image generation \cite{5cbc42c11a8dff944515a44b26e62b97d4845eff}, image restoration \cite{adrai2024deepoptimaltransportpractical} and generative modeling \cite{kwon2022scorebasedgenerativemodelingsecretly} \cite{rout2022generativemodelingoptimaltransport}. This motivates us to revisit classical mean and variance estimation methods, such as maximum likelihood estimation (MLE) within the Expectation-Maximization (EM) algorithm \cite{Efron_Hastie_2021}, and explore the advantages of utilizing optimal transport metrics, particularly the minimization of the Wasserstein distance \cite{frogner2015learningwassersteinloss}.

%{\color{red}
This paper studies a complexity-theoretic parameter-estimation question in a semidiscrete optimal-transport setting. Given samples $\bm{y}_1,\dots,\bm{y}_n$ and a known template density $f_0$, we consider estimating location and scale parameters by minimizing a Wasserstein objective between the empirical sample measure and the transformed source distribution. A related development in generative modeling is the replacement of objectives tied to divergences that can behave poorly under distribution mismatch by Wasserstein-based ones, most notably in Wasserstein GANs~\cite{Arjovsky2017WGAN,cao2019multimarginalwassersteingan}. While that setting is different from ours, it illustrates a broader shift in the literature toward Wasserstein minimization and motivates the natural complexity-theoretic question studied here: what changes when one replaces likelihood-based fitting by Wasserstein minimization? In the present setting, the attraction of the Wasserstein formulation is not that it universally supersedes likelihood-based estimation, but that under squared Euclidean cost it yields a semidiscrete dual structure that can be exploited algorithmically while remaining stable under perturbations of sample locations. Existing works on computational optimal transport provide the mathematical framework for this viewpoint~\cite{peyré2020computationaloptimaltransport}, and prior empirical studies have used Wasserstein-type objectives for parameter fitting~\cite{%8578459,
Kolouri2018SlicedWassersteinGMM}. Our contribution is different: we prove that for a concrete location-scale estimation problem, Wasserstein minimization admits polynomial-time approximation guarantees, whereas the corresponding global MLE problem is NP-hard. This hardness result is in the same spirit as earlier results showing that seemingly simple Euclidean parameter-fitting problems can encode hard combinatorial structure~\cite{Dasgupta2008HardnessKMeans,JMLR:v18:16-657}.
%}

This paper focuses on the complexity of estimating translation and shrinkage parameters, specifically $\bm{\mu} \in \mathbb{R}^l$ and $\sigma \in \mathbb{R}_{++}$, for a class of distributions that are piecewise constant over a finite number of disjoint hyperrectangles. These distributions are expressed as $\frac{1}{\sigma^l} f_0 \left( \frac{\bm{x} - \bm{\mu}}{\sigma} \right)$, where $f_0$ is a known density function in $\mathbb{R}^l$. In this framework, $\bm{\mu}$ denotes the translation parameter, and $\sigma$ controls the scaling of the distribution. The goal is to estimate these parameters from $n$ samples such that the resulting distribution closely aligns with the samples, a key objective in generative modeling tasks. %{\color{red}
In two dimensions, one may view this family as a histogram-like or gridded density supported on axis-aligned rectangles; estimating a global translation and shrinkage against an empirical point cloud is a stylized version of aligning coarse spatial densities, such as population-density maps or simplified image-registration summaries.
%}

We select this specific class of distribution because it allows us to carry out the complexity-theoretic analysis outlined in this paper. While we acknowledge that fitting adversary-chosen samples to such distributions may appear contrived, this choice is driven by the requirements of computational complexity analysis. To rigorously compare the computational complexity of parameter estimation problems, we must consider finitely specifiable distributions, as computational complexity is defined in terms of finite representations. We found piecewise constant distributions on hyperrectangles to be a suitable and universal candidate, as detailed later in Claim \ref{claim2-5}. Moreover, our analysis focuses on revealing a fundamental computational complexity barrier between two parameter estimation approaches rather than their statistical interpretation.

%{\color{red}
Parameter estimation of the type studied here is closely related to the maximization step in latent-variable fitting procedures such as the Expectation--Maximization (EM) algorithm~\cite{10.1111/j.2517-6161.1977.tb01600.x} and its optimal-transport-based variants, including Sinkhorn EM~\cite{Mena2020SinkhornEA}. Our focus, however, is not to propose another EM-type method, but to isolate the computational complexity of the underlying parameter-fitting problem. In this respect, our results compare two estimation frameworks on the same model family: for the piecewise-constant location-scale family considered here, global maximum likelihood estimation can encode an NP-hard problem, whereas Wasserstein minimization admits a polynomial-time approximation scheme for estimating $\bm{\mu}$ and $\sigma$. This complements prior uses of Wasserstein objectives for parameter fitting, such as sliced-Wasserstein methods for Gaussian mixture models~\cite{%8578459,
Kolouri2018SlicedWassersteinGMM}, and it also fits into a broader complexity landscape for transport-based optimization. For example, Altschuler and Boix-Adser\`a showed that computing Wasserstein barycenters is NP-hard in general dimension, even though barycenters can be computed in polynomial time in fixed dimension; this is a different problem from ours, but it similarly illustrates that Wasserstein-based optimization problems can undergo a sharp complexity change once the dimension is part of the input~\cite{AltschulerBoixAdsera2022BaryNPHard,AltschulerBoixAdsera2021BaryFixedDim}. Additionally, MLE for mixtures of spherical Gaussians has been shown to be NP-hard~\cite{JMLR:v18:16-657}, motivating the need to explore alternative methods to MLE. Similarly, there are hardness results for estimating optimal weights for a single neuron, as shown in ~\cite{pmlr-v151-diakonikolas22a}. Throughout, we work with the squared Euclidean transport cost; extending the same complexity analysis to other transport costs is left open.
%}

There has been extensive research on analyzing the computational complexity of optimal transport algorithms \cite{an2021efficientoptimaltransportalgorithm, dvurechensky2018computationaloptimaltransportcomplexity}. While our study does not focus on algorithmic efficiency, it presents one of the first results on parameter estimation within the optimal transport framework.

Our contributions are as follows: In Section \ref{section 2}, we develop a polynomial-time method for approximating the optimal $\bm{\mu}$ and $\sigma$ using a Wasserstein minimization framework. In Section \ref{section 3}, we establish the NP-hardness of finding $\bm{\mu}$ through Maximum Likelihood Estimation (MLE) for general distributions. Section \ref{section 4} raises several open questions, highlighting potential directions for future research. 

\section{Parameter Estimation in Wasserstein Minimization} \label{section 2}

\subsection{Setup and Explicit Formulas for Parameter Estimation}
We consider a distribution consisting of a finite union of hyperrectangles and a set of sample points. Our objective is to estimate the parameters $\bm{\mu}$ and $\sigma$ from the discrete samples to minimize the Wasserstein distance induced by these parameters. This problem can be effectively framed as a semidiscrete transport problem. In other words, the selection of $\bm{\mu}, \sigma$ is regarded as finding values of these parameters that minimize the transport distance between the sample points, regarded as a sum of delta-functions, and the underlying distributions. 
\vspace{3 mm} \newline 
We will adopt the framework established in \cite{peyré2020computationaloptimaltransport} for computational optimal transport in the semi-discrete setup. We will prove that there is a polynomial-time algorithm for estimating $\bm{\mu}$ and $\sigma$ in the semidiscrete setup. Consider the semi-discrete transportation problem where $c(\bm{x}, \bm{y}_j) \in \mathbb{R}$ represents the cost of transporting from continuous source $\bm{x} \in \mathbb{R}^l$ to $n$ discrete sinks $\bm{y}_j \in \mathbb{R}^l$, for $1 \le j \le n$, i.e., $c: \mathbb{R}^l \times \{ \bm{y}_1, \dots, \bm{y}_n \} \to \mathbb{R}_+$. The optimal transport cost can be represented as follow:
\begin{equation} \label{eq-1}
\begin{aligned}
    & \min_\pi \int_{\mathbb{R}^l} \sum_{j = 1}^n c(\bm{x}, \bm{y}_j) \, d\pi(\bm{x}, \bm{y}_j) \\
    \text{subject to} \quad & \sum_{j = 1}^n d\pi(\bm{x}, \bm{y}_j) = d\alpha(\bm{x}) \quad \forall \bm{x} \in \mathbb{R}^l, \\
    & %{\color{red}
    \int_{\mathbb{R}^l} d\pi(\bm{x}, \bm{y}_j) = \frac{1}{n} \quad \forall j = 1, \ldots, n, %}
    \\ 
    & \pi \ge 0.
\end{aligned}
\end{equation}
Here, $\pi : \mathbb{R}^l \times \{ \bm{y}_1, \dots, \bm{y}_n \} \to \mathbb{R}$ is the semidiscrete optimal transport map that measures the flow from $\bm{x}$ to $\bm{y}_i$, for $i = 1, \dots, n$. The given measure $d\alpha(\bm{x})$ represents the capacity at $\bm{x} \in \mathbb{R}^l$, and %{\color{red}
each sink has demand $\frac{1}{n}$.%}. 

%{\color{red}
Since the parameter-estimation problem studied in this paper compares a probability source distribution with the empirical sample measure, we work directly in the probability-normalized case throughout Section \ref{section 2}: the source satisfies $\int_{\mathbb{R}^l}d\alpha(\bm{x})=1$. No additional generality is needed for the argument. %} 
\\ 
The problem is feasible if these capacities are all nonnegative and satisfy the compatibility condition
\begin{equation} \label{eq-2} 
%{\color{red}
\int_{\mathbb{R}^l} d \alpha(\bm{x}) = \sum_{j = 1}^n \frac{1}{n} = 1.
%}
\end{equation} 
Given these assumptions, the problem is bounded, ensuring the existence of an optimizer. Notably, this optimizer shares the same optimal value as its dual formulation, as discussed in Villani's work on optimal transport \cite{villani2008optimal}.
In the dual form, the problem is expressed as:
\[ 
%{\color{red}
\mathcal{L}_c(\alpha) = \max_{\bm{g} \in \mathbb{R}^n} \left( \int_{\mathbb{R}^l} \bm{g}^{\bar{c}}(\bm{x}) \, d\alpha(\bm{x}) + \frac{1}{n}\sum_{j=1}^n g_j \right)
%} 
\]
where $\bm{g}^{\overline{c}}(\bm{x}) := \min_{1 \le j \le n} c(\bm{x}, \bm{y}_j) - g_j $, where $c(\bm{x}, \bm{y}_j) := \| \sigma \bm{x} -  \bm{y}_j - \bm{\mu} \|^2$. We can view this as the following infinite linear program: 
\begin{align*}
    & %{\color{red}
    \max_{\bm{g}, h} \int_{\mathbb{R}^l} h(\bm{x}) \ d \alpha(\bm{x}) + \frac{1}{n}\sum_{j = 1}^n g_j %}
    \\
    \text{subject to} \quad &h(\bm{x}) + g_j \le c(\bm{x}, \bm{y}_j) , \quad 1 \le j \le n, \quad \forall \bm{x} \in \mathbb{R}^l.
\end{align*}
Here, $\bm{g} \in \mathbb{R}^n$ and $h$ is a function $\mathbb{R}^\ell \to \mathbb{R}.$ The following theorem, established in \cite{peyré2020computationaloptimaltransport}, is included here for the sake of completeness.

\begin{theorem}
    %{\color{red}
    Let $P(\bm{\mu})$ denote the instance of the program above for choice of $\bm{\mu}$ such that $c(\bm{x}, \bm{y}_j)$ has the form $\| \bm{x} - \bm{y}_j - \bm{\mu} \|^2$. Then any optimal primal solution for $P(\bm{0})$ is also optimal for $P(\bm{\mu})$. In particular, one may choose an optimal primal plan independently of $\bm{\mu}$.
    %}
\end{theorem}
\begin{proof}
  Let $DP(\bm{\mu})$ denote the corresponding dual instance. Let $\pi^{\ast}: \mathbb{R}^l \times \{ \bm{y}_1, \bm{y}_2, \dots, \bm{y}_n \} \to \mathbb{R}$ be the optimal solution to $P(\bm{0})$ (i.e., $\bm{\mu}$ is taken to be $\bm{0}$) and let $\bm{g}^{\ast} \in \mathbb{R}^n, h^{\ast} : \mathbb{R}^l \to \mathbb{R}$ be the optimal solution for $DP(\bm{0})$. By complementary slackness on $\pi^{\ast}$ and $\bm{g}^{\ast}, h^{\ast}$, we have
    \[ \pi^{\ast}(\bm{x},\bm{y}_j) (\| \bm{x} - \bm{y}_j \|^2 - g^{\ast}_j - h^{\ast}(\bm{x})) = 0, \quad \forall \bm{x} \in \mathbb{R}^l, j \in \{ 1, 2, \dots, n \} \ \text{almost surely}. \]
    Now, fix an arbitrary $\bm{\mu}$. Take
    \[ \hat{g}_j = g_j^{\ast} + 2 \bm{\mu}^\top \bm{y}_j + \| \bm{\mu} \|^2 \quad \text{and} \quad \hat{h}(\bm{x}) = h^{\ast}(\bm{x}) - 2 \bm{\mu}^\top \bm{x}, \]
    for any $\bm{x} \in \mathbb{R}^l$ and $j \in \{ 1, 2, \dots, n \}$. We have
    \begin{align*}
        \| \bm{x} - \bm{y}_j - \bm{\mu} \|^2 - \hat{g}_j - \hat{h}(\bm{x})    &= \| \bm{x} - \bm{y}_j \|^2 + \| \bm{\mu} \|^2 - 2\bm{\mu}^\top (\bm{x} - \bm{y}_j) - \hat{g}_j - \hat{h}(\bm{x}) \\ 
        &= \| \bm{x} - \bm{y}_j \|^2 + (\|\bm{\mu} \|^2 + 2 \bm{\mu}^\top \bm{y}_j - \hat{\bm{g}}_j) + (-2 \bm{\mu}^\top \bm{x} - \hat{h}(\bm{x})) \\ 
        &= \| \bm{x} - \bm{y}_j \|^2 - g_j^{\ast} - h^{\ast}(\bm{x}).
    \end{align*}
    Since $(\bm{g}^{\ast}, h^{\ast})$ is feasible for $DP(\bm{0})$, the quantity is %{\color{red}
    nonnegative
    %} 
    and thus $(\hat{\bm{g}}, \hat{h})$ is feasible for $DP(\bm{\mu})$. This therefore gives us 
    \[ \pi^{\ast}(\bm{x}, \bm{y}_j) ( \| \bm{x} - \bm{y}_j \|^2 - g_j^{\ast} - h^{\ast}(\bm{x})) = \pi^{\ast}(\bm{x}, \bm{y}_j) (\| \bm{x} - \bm{y}_j - \bm{\mu} \|^2 - \hat{g}_j - \hat{h}(\bm{x})), \]
    which implies that by complementary slackness, $\pi^{\ast}$ is an optimal solution to $P(\bm{\mu})$ as well. 
\end{proof}
\begin{remark}
    Because of the setup, the optimal transport map is not uniquely determined: the set for which it is not determined consists of a finite union of measure-zero disjoint polytopes corresponding to the boundaries of Laguerre cells (see Equation \eqref{eq-8}) meeting hyperrectangles. However, complementary slackness holds at every point. Thus, the ``almost surely" qualifier is unnecessary in our setting, and the dual program, introduced in Equation \eqref{eq-7}, has a unique solution. We will omit such phrases for the rest of our paper. %{\color{red}
    The preceding theorem should therefore be read as an existence statement for a common optimal primal plan, not as a uniqueness statement for the primal optimizer.
    %}
\end{remark}

\begin{theorem}
    %{\color{red}
    Let $Q(\sigma)$ denote the instance of the program above for choice of $\sigma$ such that $c(\bm{x}, \bm{y}_j) = \| \sigma \bm{x} - \bm{y}_j \|^2$ for some $\sigma > 0$, any optimal primal solution for $Q(1)$ is also optimal for $Q(\sigma)$. In particular, one may choose an optimal primal plan independently of $\sigma$.
    %}
\end{theorem}
\begin{proof}
     Let $DQ(\sigma)$ denote the corresponding dual instance. Let $\pi^{\ast} : \mathbb{R}^l \times \{ \bm{y}_1,\bm{y}_2 \dots, \bm{y}_n \} \to \mathbb{R}$ be the optimal solution to $Q(1)$ and let $\bm{g}^{\ast} \in \mathbb{R}^n, h^{\ast}: \mathbb{R}^l \to \mathbb{R}$ be the optimal solution for $DQ(1)$. By complementary slackness on $\pi^{\ast}$ and $\bm{g}^{\ast}, h^{\ast}$, we have 
    \[ \pi^{\ast}(\bm{x}, \bm{y}_j) (\| \bm{x} - \bm{y}_j \|^2 - g_j^{\ast} - h^{\ast}(\bm{x})) = 0, \quad \forall \bm{x} \in \mathbb{R}^l, j \in \{ 1, 2, \dots, n \}.\]
    Now, fix an arbitrary $\sigma \in \mathbb{R}_{++}$. Take
    \[ \hat{g}_j := (1 - \sigma) \| \bm{y}_j \|^2 + \sigma \cdot g_j^{\ast} \quad \text{and} \quad \hat{h}(\bm{x}) := (\sigma^2 - \sigma) \| \bm{x} \|^2 + \sigma \cdot h^{\ast}(\bm{x}) \]
    for any $\bm{x} \in \mathbb{R}^l$ and $j \in \{ 1, 2, \dots, n \}$. We have 
    \begin{align*}
        \| \sigma \bm{x} - \bm{y}_j \|^2 - \hat{g}_j - \hat{h}(\bm{x}) &= \sigma\| \bm{x} - \bm{y}_j \|^2 - (\hat{g}_j - (1 - \sigma) \| \bm{y}_j \|^2 )  - (\hat{h}(\bm{x}) - (\sigma^2 - \sigma) \| \bm{x} \|^2)  \\ 
        &= \sigma(\| \bm{x} - \bm{y}_j \|^2 - g_j^{\ast} - h^{\ast}(\bm{x})).
    \end{align*}
    Since $(\bm{g}^{\ast}, h^{\ast})$ is feasible for $DQ(1)$, the quantity is %{\color{red}
    nonnegative
    %} 
    and thus $(\hat{\bm{g}}, \hat{h})$ is feasible for $DQ(\sigma)$. This therefore gives us
    \[ \sigma \cdot \pi^{\ast}(\bm{x}, \bm{y}_j)(\| \bm{x} - \bm{y}_j \|^2 - g_j^{\ast} - h^{\ast}(\bm{x})) = \pi^{\ast}(\bm{x}, \bm{y}_j)(\| \sigma \bm{x} - \bm{y}_j \|^2 - \hat{g}_j - \hat{h}(\bm{x})) \]
     which implies that by complementary slackness, $\pi^{\ast}$ is an optimal solution to $Q(\sigma)$ as well. 
\end{proof}

Combining both of the theorem above give us the following corollary. 
\begin{corollary}
    %{\color{red}
    Suppose that there exists  $\bm{y}_1, \dots, \bm{y}_n \in \mathbb{R}^l$ and $\bm{\mu} \in \mathbb{R}^l, \sigma \in \mathbb{R}_{++}$ such that for all $\bm{x} \in \mathbb{R}^l, j \in \{ 1, \dots, n \}$, $c(\bm{x}, \bm{y}_j) = \| \sigma \bm{x} - \bm{y}_j - \bm{\mu} \|^2$. Then one may choose an optimal primal solution $\pi$ independently of $\bm{\mu}$ and $\sigma$.
    %}
\end{corollary}

We note that if $\pi$ is the optimal transport map, then the optimal objective value of \eqref{eq-1} for general $\bm{\mu} \in \mathbb{R}^l, \sigma \in \mathbb{R}$ is
\[ f(\bm{\mu}, \sigma) = \int_{\mathbb{R}^l } \sum_{j = 1}^n \| \sigma \bm{ x} - \bm{y}_j - \bm{\mu} \|^2 d\pi(x,y_j)  \]

\begin{theorem}
    The optimal choices of $\bm{\mu}^{\ast}$ and $\sigma^{\ast}$ are given by the formulas:
    \begin{align}
        \sigma^{\ast} &= \frac{\left( \int_{\mathbb{R}^l} \sum_{j=1}^n \bm{x}^\top \bm{y}_j \, d\pi(\bm{x}, \bm{y}_j) \right) - \left( \frac{1}{n}\sum_{j=1}^n \bm{y}_j \right)^\top \left( \int_{\mathbb{R}^l} \bm{x} \, d\alpha(\bm{x}) \right)}{\left( \int_{\mathbb{R}^l} \| \bm{x} \|^2 \, d\alpha(\bm{x}) \right) - \left\| \int_{\mathbb{R}^l} \bm{x} \, d\alpha(\bm{x}) \right\|^2} \label{eq-3} \\ 
        \bm{\mu}^{\ast} &= \sigma^{\ast} \int_{\mathbb{R}^l} \bm{x} \, d\alpha(\bm{x}) - \frac{1}{n}\sum_{j=1}^n \bm{y}_j, \label{eq-4}
    \end{align}
    %{\color{red}
    where the normalization \eqref{eq-2} gives total mass $1$.
    %} 
\end{theorem}

\begin{proof}
    We will find $(\bm{\mu}^{\ast}, \sigma^{\ast}) \in \argmin_{\bm{\mu}, \sigma} f(\bm{\mu}, \sigma)$. To do this, note that $f$ is a convex quadratic function of $\bm{\mu}$ and $\sigma$. As $(\bm{\mu}^{\ast} ,\sigma^{\ast})$ is a critical point of $f$, we have 
    \[ \bm{0} = \frac{\partial f}{\partial \bm{\mu}} (\bm{\mu}, \sigma) \Big|_{\bm{\mu}^{\ast}, \sigma^{\ast}} = \int_{\mathbb{R}^l} \sum_{j=1}^n 2(\bm{\mu}^{\ast} + \bm{y}_j - \sigma^{\ast} \bm{x}) \, d\pi(\bm{x}, \bm{y}_j) \]
    \[ 0 = \frac{\partial f}{\partial \sigma} (\bm{\mu}, \sigma) \Big|_{\bm{\mu}^{\ast}, \sigma^{\ast}} = \int_{\mathbb{R}^l} \sum_{j=1}^n 2 \bm{x}^\top (\sigma^{\ast} \bm{x} - \bm{y}_j - \bm{\mu}^{\ast}) \, d\pi(\bm{x}, \bm{y}_j) \]
    This yields the following two linear equations with respect to $\bm{\mu}^{\ast}$ and $\sigma^{\ast}$:
    \begin{align*}
        \bm{\mu}^{\ast} \int_{\mathbb{R}^l} \sum_{j=1}^n 1 \, d\pi(\bm{x}, \bm{y}_j) - \sigma^{\ast} \int_{\mathbb{R}^l} \sum_{j=1}^n \bm{x} \, d\pi(\bm{x}, \bm{y}_j) + \int_{\mathbb{R}^l \times \{ 1, \dots, n \}} \bm{y}_j \, d\pi(\bm{x}, \bm{y}_j) &= \bm{0} \\
        -(\bm{\mu}^{\ast})^\top \int_{\mathbb{R}^l} \sum_{j=1}^n \bm{x} \, d\pi(\bm{x}, \bm{y}_j) + \sigma^{\ast} \int_{\mathbb{R}^l} \sum_{j=1}^n \| \bm{x} \|^2 \, d\pi(\bm{x}, \bm{y}_j) - \int_{\mathbb{R}^l} \sum_{j=1}^n \bm{x}^\top \bm{y}_j \, d\pi(\bm{x}, \bm{y}_j) &= 0
    \end{align*}
    Solving these equations, we obtain
    \begin{align*}
        \sigma^{\ast} &= \frac{\left( \int_{\mathbb{R}^l} \sum_{j=1}^n \bm{x}^\top \bm{y}_j \, d\pi(\bm{x}, \bm{y}_j) \right) - \left( \frac{1}{n}\sum_{j=1}^n \bm{y}_j \right)^\top \left( \int_{\mathbb{R}^l} \bm{x} \, d\alpha(\bm{x}) \right)}{\left( \int_{\mathbb{R}^l} \| \bm{x} \|^2 \, d\alpha(\bm{x}) \right) - \left\| \int_{\mathbb{R}^l} \bm{x} \, d\alpha(\bm{x}) \right\|^2} \\ 
        \bm{\mu}^{\ast} &= \sigma^{\ast} \int_{\mathbb{R}^l} \bm{x} \, d\alpha(\bm{x}) - \frac{1}{n}\sum_{j=1}^n \bm{y}_j,
    \end{align*}
    as desired. 
\end{proof}

The theorem provides explicit optimal formulas for $\bm{\mu}^{\ast}$ and $\sigma^{\ast}$. Most of the terms appearing in the above formulas can be computed directly. However, the integral
\[ \int_{\mathbb{R}^l} \sum_{j=1}^n \bm{x}^\top \bm{y}_j \, d\pi(\bm{x},\bm{y}_j) \]
requires knowledge of the optimal primal solution, making it apparently computationally infeasible to determine in polynomial time. Consequently at first glance, finding the exact values of $\bm{\mu}^{\ast}$ and $\sigma^{\ast}$ appears to be intractable. We resolve this issue in forthcoming sections. 

\subsection{Polynomial-Time Algorithm for Parameter Estimation}

The main theorem of this section, Theorem \ref{thm2-7}, asserts that $\bm{\mu}^{\ast}$ and $\sigma^{\ast}$ can be estimated in polynomial time under the assumption that the source distribution $d\alpha(\bm{x})$ is piecewise constant over a finite number of hyperrectangles.
%{\color{red}
%}

 For our analysis, as mentioned in Section 1, we will henceforth assume that the source distribution $d\alpha(\bm{x})$ comes from $\mathcal{C}$, where $\mathcal{C}$ denote the class of probability distributions that are piecewise constant over a finite number of hyperrectangles, i.e., there exists $k \in \mathbb{N}$, $k$ disjoint hyperrectangles $\mathcal{H}_1, \mathcal{H}_2, \dots, \mathcal{H}_k \subseteq \mathbb{R}^l$ and $k$ positive reals $\gamma_1, \dots, \gamma_k$ for which 
\[ d\alpha(\bm{x}) = \begin{cases} \gamma_i \ d\bm{x}, &\text{if } \bm{x} \in \mathcal{H}_i \text{ for some } 1 \le i \le k, \\ 0, &\text{otherwise}. \end{cases} \]
As $\alpha$ is a probability distribution, it follows that $\sum_{i = 1}^{k} \gamma_i \cdot \text{vol}(\mathcal{H}_i) = 1$.  Let $\mathcal{F}$ denote the class of probability distributions with density function specified by a nonnegative Riemann integrable function $\varphi$, where $\varphi : \mathbb{R}^l \to \mathbb{R}$ satisfies $\int_{\mathbb{R}^l} \varphi(\bm{x}) \, d\bm{x} = 1$. In fact, the general class of distributions $\mathcal{F}$ can be approximated arbitrarily closely by the class $\mathcal{C}$ by definition of Riemann integrability \cite{adams2013calculus}. 

\begin{claim} \label{claim2-5}
    For any probability distribution $f \in \mathcal{F}$ and any $\varepsilon > 0$, there exists $g \in \mathcal{C}$ such that 
    \[ \int_{\mathbb{R}^l} |f(\bm{x}) - g(\bm{x})| \, d\bm{x} < \varepsilon. \]
\end{claim}

This motivates our focus on the class of distributions $\mathcal{C}$ for our algorithm design. While the class $\mathcal{C}$ provides sufficient generality for our purposes, extending the algorithm to the broader function class $\mathcal{F}$ is not immediate due to Claim \ref{claim2-5}'s implicit exponential dependence on $l$. %{\color{red}
Thus, Theorem \ref{thm2-7} should be interpreted as polynomial in the explicit geometric description complexity of the source measure rather than as a dimension-free tractability result for arbitrary high-dimensional densities. In particular, if one approximates a general Riemann-integrable density by hyperrectangles, then $k$ may itself scale exponentially with the ambient dimension.
%}
\vspace{3 mm} \newline 
We now highlight several key parameters that play a crucial role in our analysis, which dictates the runtime of our parameter estimation algorithm for Wasserstein minimization. For convenience, we define a \textbf{reference set} $\mathcal{X}\subseteq \mathbb{R}^l$, which consists of all corners of the $k$ hyperrectangles, along with all sample points. Specifically, if each of the $k$ hyperrectangles is described by $\mathcal{H}_i = [a_{i,1}, b_{i,1}] \times [a_{i,2}, b_{i,2}] \times \cdots \times [a_{i,l}, b_{i,l}]$, we define 
\begin{align*}
    \mathcal{X}_1 &= \{\bm{y}_j\}_{j=1}^n \\ 
    \mathcal{X}_2 &= \bigcup_{i=1}^k \left\{ (x_{i,1}, x_{i,2}, \dots, x_{i,l}) \ \bigg| \ x_{i,h} \in \{ a_{i,h}, b_{i,h} \} \text{ for each } h = 1, 2, \dots, l \right\} \\
    \mathcal{X} &= \mathcal{X}_1 \cup \mathcal{X}_2 
\end{align*}
and thus $\mathcal{X}$ is a set of size at most $k \cdot 2^l + n$. 

\begin{theorem} \label{thm2-7}
Let $n$ be the number of sample points, $l$ be the dimensionality of the space, $k$ be the number of hyperrectangles characterizing the source distribution, $D = \max_{\bm{x} \in \mathcal{X}} \| \bm{x} \|$ and $s = \min \{ \min_{\substack{\bm{x}, \bm{y} \in \mathcal{X}_1 \\ \bm{x} \not= \bm{y}}} \| \bm{x} - \bm{y} \|, \min_{\substack{1 \le i \le k \\ 1 \le j \le l}} |b_{i,j} - a_{i,j}| \}$.
There exists a randomized polynomial-time algorithm with a runtime of 
\[ \text{poly}\left( n, l,  k,  D, \frac{1}{s}, \frac{1}{\varepsilon}, \log \frac{1}{\eta} \right) \]
that estimates $\bm{\mu}^{\ast}$ within $\varepsilon D$-accuracy and $\sigma^{\ast}$ within $\varepsilon$-accuracy with probability at least $1 - \eta$.
\end{theorem}

\begin{remark}
    We first remark that as $D = \max_{\bm{x} \in \mathcal{X}} \| \bm{x} \|$, we have
    \[ \max_{\bm{x}, \bm{y} \in \mathcal{X}} \| \bm{x} - \bm{y} \| \le 2D. \]
\end{remark}

The proof of Theorem~\ref{thm2-7} proceeds in four steps. First, using the explicit formulas of Section~2.1, we reduce parameter estimation to approximating a single semidiscrete transport quantity. Second, we prove that the corresponding dual objective is smooth on the relevant subspace, with an explicit Lipschitz constant. Third, because the source density is piecewise constant on hyperrectangles, we show that approximate gradients can be obtained in randomized polynomial time via volume estimation of Laguerre-cell intersections. Finally, we combine these ingredients with an inexact gradient-descent argument to obtain an $\varepsilon$-accurate estimate of the optimal parameters with high probability.

\begin{proof} 
From equations \eqref{eq-3} and \eqref{eq-4}, the optimal values of $\bm{\mu}^\ast$ and $\sigma^\ast$ can be expressed as:  
\begin{align*}
    \sigma^\ast &= \frac{\left( \int_{\mathbb{R}^l} \sum_{j=1}^n \bm{x}^\top \bm{y}_j \, d\pi(\bm{x}, \bm{y}_j) \right) - \left( \frac{1}{n}\sum_{j=1}^n \bm{y}_j \right)^\top \left( \int_{\mathbb{R}^l} \bm{x} \, d\alpha(\bm{x}) \right)}{\left( \int_{\mathbb{R}^l} \| \bm{x} \|^2 \, d\alpha(\bm{x}) \right) - \left\| \int_{\mathbb{R}^l} \bm{x} \, d\alpha(\bm{x}) \right\|^2}, \\ 
    \bm{\mu}^\ast &= \sigma^\ast \int_{\mathbb{R}^l} \bm{x} \, d\alpha(\bm{x}) - \frac{1}{n}\sum_{j=1}^n \bm{y}_j.
\end{align*}  
The terms $\int_{\mathbb{R}^l} \| \bm{x} \|^2 \, d\alpha(\bm{x})$, $\int_{\mathbb{R}^l} \bm{x} \, d\alpha(\bm{x})$, and $\frac{1}{n}\sum_{j=1}^n \bm{y}_j$ can all be computed explicitly.  
To compute $\int_{\mathbb{R}^l} \| \bm{x} \|^2 \, d \alpha(\bm{x})$, observe that this integral can be expressed as $\sum_{i = 1}^k \gamma_i \int_{\mathcal{H}_i} \| \bm{x} \|^2 \, d\bm{x}$. By denoting $\mathcal{H}_i = [a_{i,1}, b_{i,1}] \times [a_{i,2}, b_{i,2}] \times \cdots \times [a_{i,l}, b_{i,l}]$ for each $i$, we have
\begin{align*}
\int_{\mathcal{H}_i} \| \bm{x} \|^2 \, d\bm{x} &= \int_{a_{i,1}}^{b_{i,1}} \int_{a_{i,2}}^{b_{i,2}} \cdots \int_{a_{i,l}}^{b_{i,l}} (x_1^2 + x_2^2 + \cdots + x_l^2) \, dx_1 \, dx_2 \cdots dx_l \\ 
&= \frac{1}{3} \sum_{j = 1}^l (b_{i,j}^3 - a_{i,j}^3) \prod_{h \ne j} (b_{i,h} - a_{i,h}).
\end{align*}
Thus, $\int_{\mathbb{R}^l} \| \bm{x} \|^2 \, d\alpha(\bm{x})$ can be computed explicitly as
\[ \int_{\mathbb{R}^l} \| \bm{x} \|^2 \, d\alpha(\bm{x}) = \frac{1}{3} \sum_{i=1}^k \gamma_i \sum_{j=1}^l (b_{i,j}^3 - a_{i,j}^3) \prod_{h \ne j} (b_{i,h} - a_{i,h}). \]  
The remaining terms, $\int_{\mathbb{R}^l} \bm{x} \, d\alpha(\bm{x})$ and $\frac{1}{n}\sum_{j=1}^n \bm{y}_j$, can be evaluated in a similar manner.  

To compute $\int_{\mathbb{R}^l} \sum_{j=1}^n \bm{x}^\top \bm{y}_j \, d\pi(\bm{x}, \bm{y}_j)$, consider the case where $\bm{\mu} = \bm{0}$ and $\sigma = 1$, with $\pi$ as the primal optimal solution. The corresponding primal optimal cost is:  
\[ p^\ast = \int_{\mathbb{R}^l} \sum_{j=1}^n \| \bm{x} - \bm{y}_j \|^2 \, d\pi(\bm{x}, \bm{y}_j). \]  
Expanding and substituting the constraints in \eqref{eq-1}, we find:  
\[ p^\ast = \int_{\mathbb{R}^l} \| \bm{x} \|^2 \, d\alpha(\bm{x}) + \frac{1}{n}\sum_{j=1}^n \| \bm{y}_j \|^2 - 2 \int_{\mathbb{R}^l} \sum_{j=1}^n \bm{x}^\top \bm{y}_j \, d\pi(\bm{x}, \bm{y}_j). \]  
Thus, $\int_{\mathbb{R}^l} \sum_{j=1}^n \bm{x}^\top \bm{y}_j \, d\pi(\bm{x}, \bm{y}_j)$ can be computed directly from $p^\ast$, which does not need $\pi$. 
\vspace{3 mm} \newline 
By the Strong Duality Theorem, the optimal primal cost $p^{\ast}$ is equal to the optimal value of the dual program:
\begin{equation} \label{eq-5}
\begin{aligned}
    & \max_{\bm{g}, h}  \int_{\mathbb{R}^l} h(\bm{x}) \, d\alpha(\bm{x}) + %{\color{red}
    \frac{1}{n}\sum_{j=1}^n g_j
    %}
    \\
    \text{subject to} \quad & g_j + h(\bm{x}) \le c(\bm{x}, \bm{y_j}), \quad \forall \bm{x} \in \mathbb{R}^l, \quad \forall j = 1, \dots, n.
\end{aligned}
\end{equation}
 It suffices to show that we can estimate the optimal value $p^{\ast}$ to the dual program sufficiently close enough efficiently for $\bm{\mu} = \bm{0}$ and $\sigma = 1$ (the cost function of the dual program is thus $c(\bm{x}, \bm{y}_j) = \| \bm{x} - \bm{y}_j \|^2$) in polynomial time.  To achieve this, we need to maximize the 
 %energy function $\mathcal{E}(\bm{g})$ 
 dual problem in polynomial time. Following the definition in \cite{peyré2020computationaloptimaltransport}, we may define the function the energy function $\mathcal{E}(h, \bm{g})$ as
\begin{equation} \label{7}
\mathcal{E}(h, \bm{g}) := \int_{\mathbb{R}^l} h(\bm{x}) \, d \alpha(\bm{x}) + 
%{\color{red}
\frac{1}{n}\sum_{j=1}^n g_j
%} 
- \iota_{\Gamma}(h, \bm{g}), 
\end{equation} 
where $\Gamma$ is the feasible region of \eqref{eq-5} and the indicator function $\iota_{\Gamma}$ is defined as: 
\[
\iota_{\Gamma}(h, \bm{g}) = \begin{cases} 
0 & \text{if } g_i + h(\bm{x}) \le c(\bm{x}, \bm{y_j}), \quad\forall \bm{x} \in \mathbb{R}^l, \quad \forall j = 1, \dots, n , \\ 
+\infty & \text{otherwise}.
\end{cases}
\]
It follows that $\mathcal{E}(h, \bm{g})$ is a concave function with respect to both $h$ and $\bm{g}$. We further define $\mathcal{E}(\bm{g}) = \max_{h} \mathcal{E}(h, \bm{g})$. By Proposition 8.35 in \cite{bauschke2011convex}, it follows that $\mathcal{E}(\bm{g})$ is also concave.
\vspace{3 mm} \newline 
To explicitly derive $\mathcal{E}(\bm{g})$, observe that $h$ is weighted nonnegatively in the definition of $\mathcal{E}(h, \bm{g})$. Therefore, $\max_{h} \mathcal{E}(h, \bm{g})$ is obtained by setting $h$ pointwise to its maximum value while satisfying the constraints. Specifically, for each $\bm{x}$, $h(\bm{x})$ should be chosen such that $h(\bm{x}) + g_j \leq c(\bm{x}, \bm{y}_j)$ for all $j$, with at least one $j$ where equality holds: $h(\bm{x}) + g_j = c(\bm{x}, \bm{y}_j)$. Consequently, $h(\bm{x})$ should be set to $\min_j (c(\bm{x}, \bm{y}_j) - g_j)$ for all $\bm{x}$.
 \vspace{3 mm} \newline 
 This allows us to write 
 %the minimized energy 
 $\mathcal{E}(\bm{g})$ for any $\bm{g} \in \mathbb{R}^n$ explicitly as
\begin{equation} \label{eq-7}
\mathcal{E}(\bm{g}) = \sum_{j = 1}^n \int_{\mathbb{L}_j(\bm{g})} \left(c(\bm{x}, \bm{y}_j) - g_j \right) \, d \alpha(\bm{x}) + \langle \bm{g}, \frac{1}{n} \bm{1} \rangle, 
\end{equation}
where $c(\bm{x}, \bm{y}_j) \in \mathbb{R}$ is the cost function, and in this paper, we specifically consider $c(\bm{x}, \bm{y}_j) = \| \bm{x} - \bm{y}_j \|^2$. Here, $\bm{1}$ denote the all-one vector, and $\mathbb{L}_j(\bm{g})$ denotes the \textit{Laguerre cell} associated with the dual weights $\bm{g}$, i.e.,
   \begin{equation} \label{eq-8} \mathbb{L}_j(\bm{g}) = \{ \bm{x} \in \mathbb{R}^l : \forall j' \not= j, \| \bm{x} - \bm{y}_j \|^2 - g_j \le \| \bm{x} - \bm{y}_{j'} \|^2 - g_{j'} \} \end{equation}
  which induces a decomposition of $\mathbb{R}^l$ as $\mathbb{R}^l = \bigcup_{1 \le j \le n} \mathbb{L}_j(\bm{g})$, such that $\text{int}(\mathbb{L}_j)$ are pairwise disjoint for $1 \le j \le n$. In the case $\bm{g} = \bm{0}$, Laguerre cells are commonly called Voronoi cells. 
    \vspace{3 mm} \newline 
    To maximize the energy $\mathcal{E}(\bm{g})$, we can apply inexact gradient ascent as the energy function $\mathcal{E}$ is concave as established. The function $\mathcal{E}(\bm{g})$ is differentiable and the gradient is given by \cite{peyre_courseot} as:
    \[ 
    %{\color{red}
    \nabla \mathcal{E}(\bm{g})_j = -\int_{\mathbb{L}_j(\bm{g})} d\alpha(\bm{x}) + \frac{1}{n}, \ \forall j \in [n].
    %} 
    \]
    Under our source distribution's assumption, we can thus rewrite $\nabla \mathcal{E}(\bm{g})_j$ as
    \[ 
    %{\color{red}
    \nabla \mathcal{E}(\bm{g})_j = \sum_{\ell = 1}^k -\gamma_{\ell} \cdot \text{vol}(\mathbb{L}_j(\bm{g}) \cap \mathcal{H}_{\ell}) + \frac{1}{n}.
    %} 
    \]
     Note that for every $1 \le \ell \le k$ and $1 \le j \le n$, $\mathbb{L}_j(\bm{g}) \cap \mathcal{H}_{\ell}$ is a convex body for which we can provide a separation oracle efficiently. By using Kannan, Lovasz, and Simonovits' Theorem which is stated in Theorem \ref{oracle}, given $\varepsilon, \eta > 0$, we can estimate $\text{vol}(\mathbb{L}_j(\bm{g}) \cap \mathcal{H}_{\ell})$ within $\varepsilon$ accuracy with probability at least $1 - \eta$ in polynomial time. The original $\tilde{O}(n^5)$ algorithm has since been improved to $\tilde{O}(n^4)$ by \cite{LOVASZ2006392}.

    \begin{theorem}[Kannan, Lovasz, Simonovits \cite{Kannan1997RandomWA}] \label{oracle}
    Given a separation oracle of an $n$-dimensional convex body $K$, and $\varepsilon, \eta > 0$, there is an algorithm that uses $O \left[ \text{poly} \left( n, \frac{1}{\varepsilon}, \log \frac{1}{\eta} \right) \right]$ oracle calls which returns a real number $\zeta$ for which
    \[ \left| \frac{\text{vol}(K)}{\zeta} - 1 \right| < \varepsilon  \]
    with probability at least $1 - \eta$. 
    \end{theorem}
    
    We now present the separation oracles for a fixed Laguerre cell, which, when combined with the theorem above, provide a method to compute $\nabla \mathcal{E}(\bm{g})$ to within $\varepsilon$-accuracy in polynomial time. The separation oracle for a hyperrectangle is straightforward, as it can be determined by $O(l)$ min/max operations. For Laguerre cells, the separation oracle can be constructed as follows: Given $\bm{y}_1, \dots, \bm{y}_n$ and $g_1, \dots, g_n$, where $\bm{y}_i \in \mathbb{R}^l$ and $g_i \in \mathbb{R}$ for all $1 \le i \le n$, recall the definition of Laguerre cells associated with $\bm{g}$, as in \eqref{eq-8}. 
    We will construct a separation oracle for a specific Laguerre cell $\mathbb{L}_i(\bm{g})$. Check whether $\| \bm{x} - \bm{y}_i \|^2 - g_i \le \| \bm{x} - \bm{y}_j \|^2 - g_j$ for $1 \le j \le n$ as follow: 
    \begin{itemize}
        \item If $\| \bm{x} - \bm{y}_i \|^2 - g_i \le \| \bm{x} - \bm{y}_j \|^2 - g_j$ for all $1 \le j \le n$, this means that $\bm{x} \in \mathbb{L}_i(\bm{g})$ and the oracle will confirm that $\bm{x}$ is inside the Laguerre cell $\mathbb{L}_i(\bm{g})$.  
        \item Otherwise, pick the smallest $1 \le j \le n$ where the inequality is violated, i.e. we have $\| \bm{x} - \bm{y}_j \|^2 - g_j < \| \bm{x} - \bm{y}_i \|^2 - g_i$, which is equivalent to $2(\bm{y}_j - \bm{y}_i)^\top \bm{x} > g_i - g_j + \| \bm{y}_j \|^2 - \| \bm{y}_i \|^2$ and thus by picking the hyperplane $\bm{\alpha}^\top \bm{x} = \beta$ where $\bm{\alpha} = 2(\bm{y}_j - \bm{y}_i)$ and $\beta = g_i - g_j + \| \bm{y}_j \|^2 - \| \bm{y}_i \|^2$, we get a hyperplane that separates $\bm{x}$ from the Laguerre cell $\mathbb{L}_i(\bm{g})$, as desired. 
    \end{itemize} 

This establishes a separation oracle for a hyperrectangle that operates in $O(l)$ time and a separation oracle for a Laguerre cell that operates in $O(ln)$ time. 

To establish an accurate complexity bound for applying inexact gradient descent to minimize the convex function $-\mathcal{E}$, we will first prove that $\mathcal{E}$ is $L$-smooth, where $ L = \text{poly} \left( n, l,  k, \frac{1}{s}\right) $ where $n,l,k,s$ are defined in Theorem \ref{thm2-7}.
Our proof begins by noting an important inequality regarding the ratio of the volume of a hyperrectangle and its projection to some subspace of dimension $l - 1$.

\begin{lemma}\label{projection}
       Let $\mathcal{V} \in \mathbb{R}^l$ be a hyperrectangle with minimum width $\xi$, and let $\mathcal{Z}$ be the projection of $\mathcal{H}$ into an arbitrary subspace of dimension $l - 1$. Then, 
       \[ \text{vol}(\mathcal{Z}) \le \frac{2l}{\xi} \cdot  \text{vol}(\mathcal{V}) .\]
\end{lemma}

   \begin{proof} 
   Let $\mathcal{W}$ be the facet of $\mathcal{V}$ with the largest surface area. We know that $\text{vol}(\mathcal{W}) \cdot \xi = \text{vol}(\mathcal{V})$, i.e. we have $\text{vol}(\mathcal{W}) = \frac{1}{\xi} \cdot \text{vol}(\mathcal{V})$. We note that any orthogonal projection of the $(l - 1)$-facet $\mathcal{W}$ onto a hyperplane can only decrease its volume because the projection of a $(l - 1)$-facet is a hyperparallelogram, and the volume of an $(l - 1)$-parallelogram is bounded by the product of the lengths of its sides by Hadamard's inequality. However, the side lengths of $\mathcal{W}$ can only shrink under projection. This, thus shows us that $\text{vol}(\text{proj}_\mathcal{P}(\mathcal{W})) \le \frac{1}{\xi} \text{vol}(\mathcal{V})$ for any hyperplane $\mathcal{P}$. 
   \vspace{3 mm} \newline 
   Finally, we observe that $\text{proj}_\mathcal{P}(\mathcal{V}) \subseteq \bigcup_i \text{proj}_{\mathcal{P}}(\mathcal{W}_i)$ where this union of $\mathcal{W}_i$ ranges over all possible $2l$ facets. To see this, let $\bm{x} \in \text{int}(\mathcal{V})$ and let $\bm{a}$ be the normal to the hyperplane $\mathcal{P}$. Then the segment from $\text{proj}_\mathcal{P}(\bm{x})$ to $\bm{x}$ is parallel to $\bm{a}$ by definition of projection. However, extending this segment to a line must pass through the boundary of $\mathcal{V}$ twice, so there are two points on $\partial \mathcal{V}$ that have the same projection as $\bm{x}$. 
       \vspace{3 mm} \newline 
This thus shows that $\text{vol}(\text{proj}_{\mathcal{P}}(\mathcal{V})) \leq \frac{2l}{\xi} \cdot \text{vol}(\mathcal{V})$ for any hyperplane $\mathcal{P}$, i.e. we must have $\text{vol}(\mathcal{Z}) \le \frac{2l}{\xi} \cdot \text{vol}(\mathcal{V})$. 
   \end{proof}

   By definition of $s$ in Theorem \ref{thm2-7}, we have that the claim above holds even by replacing $\xi$ with $s$, as we know that $\xi \ge s$.  Before proving $\mathcal{E}$ is $L$-smooth, we will prove a crucial claim regarding the square of the difference in volumes of Laguerre cells generated by two functions within a finite hyperrectangle. This claim will be the core of the proof of $L$-smoothness. 

\begin{lemma} \label{hyperrec}
       Given any finite hyperrectangle $\mathcal{H}$, and any $\bm{g}, \bm{h} \in \mathbb{R}^n$, we have
        \begin{equation} \label{lem2-11}
        \sum_{j = 1}^n (\text{vol}(\mathbb{L}_j(\bm{g}) \cap \mathcal{H}) - \text{vol}(\mathbb{L}_j(\bm{h}) \cap \mathcal{H}))^2 \le \frac{4n^2l^2}{s^4} \text{vol}(\mathcal{H})^2 \sum_{j = 1}^n |g_j - h_j|^2.
        \end{equation} 
\end{lemma}

    \begin{proof}
        The key idea of this proof is to analyze the sum coordinate-wise. Consider $\bm{f}_0, \bm{f}_1, \dots, \bm{f}_n \in \mathbb{R}^n$ where $\bm{f}_0 = \bm{g}, \bm{f}_n = \bm{h}$, and for any $1 \le j \le n -  1$, we have
        \[ (\bm{f}_i)_j= \begin{cases} h_j &\text{if } 1 \le j \le i \\ g_j &\text{otherwise.} \end{cases}.\]
        We first bound the LHS of \eqref{lem2-11} when $\bm{g} := \bm{f}_{\ell - 1}$ and $\bm{h} := \bm{f}_{\ell}$ for any $1 \le \ell \le n $. Fix any $1 \le \ell \le n$. By definition of $(\bm{f}_i)_{0 \le i \le n}$, this reduces to proving 
        \[ \sum_{j = 1}^n (\text{vol}(\mathbb{L}_j(\bm{f}_{\ell - 1}) \cap \mathcal{H}) - \text{vol}(\mathbb{L}_j(\bm{f}_{\ell}) \cap \mathcal{H}))^2 \le c_\ell (g_{\ell} - h_{\ell})^2\text{vol}(\mathcal{H})^2  \]
        for some $c_{\ell} \in \text{poly} \left( n, l, \frac{1}{s} \right)$. We note that by definition, we have
        \begin{align*}
            \mathbb{L}_j(\bm{f}_{\ell - 1}) \cap \mathcal{H} &= \{ \bm{x} \in \mathcal{H}: \| \bm{x} - \bm{y}_j \|^2 - (\bm{f}_{\ell - 1})_j \le \| \bm{x} - \bm{y}_i \|^2 - (\bm{f}_{\ell - 1})_i \ \forall i \}, \\ 
            \mathbb{L}_j(\bm{f}_{\ell}) \cap \mathcal{H} &= \{ \bm{x} \in \mathcal{H}: \| \bm{x} - \bm{y}_j \|^2  - (\bm{f}_{\ell})_j \le \| \bm{x} - \bm{y}_i \|^2 - (\bm{f}_{\ell})_i \  \forall i \}.
        \end{align*}
        By construction, $(\bm{f}_{\ell})_i = (\bm{f}_{\ell - 1})_i$ for all $i \not= \ell$. Without loss of generality we may assume $(\bm{f}_{\ell - 1})_{\ell} = g_{\ell} \ge h_{\ell} = (\bm{f}_{\ell})_{\ell}$. We claim that
        \[ \mathbb{L}_{\ell}(\bm{f}_{\ell}) \cap \mathcal{H} \subseteq \mathbb{L}_{\ell}(\bm{f}_{\ell - 1}) \cap \mathcal{H} \quad \text{and} \quad \mathbb{L}_i(\bm{f}_{\ell}) \cap \mathcal{H} \supseteq \mathbb{L}_i(\bm{f}_{\ell - 1}) \cap \mathcal{H} \ \forall i \not= \ell \]
This is because for any $\bm{x} \in \mathbb{L}_{\ell}(\bm{f}_{\ell}) \cap \mathcal{H}$, we have
\[ \| \bm{x} - \bm{y}_{\ell} \|^2 - (\bm{f}_{\ell - 1})_{\ell} \le \| \bm{x} - \bm{y}_{\ell} \|^2 - (\bm{f}_{\ell})_{\ell}  \le \| \bm{x} - \bm{y}_i \|^2 - (\bm{f}_{\ell})_i = \| \bm{x} - \bm{y}_i \|^2 - (\bm{f}_{\ell - 1})_i \]
for all $i \not= \ell$, and thus $\bm{x} \in \mathbb{L}_{\ell}(\bm{f}_{\ell - 1}) \cap \mathcal{H}$. The other inclusion follows similarly. Having the above inclusion in mind, we can simplify the volume difference as
        \begin{align*} 
        (\text{vol}(\mathbb{L}_{\ell}(\bm{f}_{\ell - 1}) \cap \mathcal{H}) - \text{vol}(\mathbb{L}_{\ell}(\bm{f}_{\ell}) \cap \mathcal{H}))^2 &=  \text{vol}((\mathbb{L}_{\ell}(\bm{f}_{\ell - 1}) \setminus \mathbb{L}_{\ell}(\bm{f}_{\ell})) \cap \mathcal{H})^2, \\ 
        (\text{vol}(\mathbb{L}_{j}(\bm{f}_{\ell - 1}) \cap \mathcal{H}) - \text{vol}(\mathbb{L}_{j}(\bm{f}_{\ell}) \cap \mathcal{H}))^2 &= \text{vol}((\mathbb{L}_j(\bm{f}_{\ell}) \setminus \mathbb{L}_j(\bm{f}_{\ell - 1})) \cap \mathcal{H})^2  \quad \forall j \not= \ell. 
        \end{align*} 
        Now, we note that $\bm{x} \in \mathbb{L}_{\ell}(\bm{f}_{\ell - 1}) \setminus \mathbb{L}_{\ell}(\bm{f}_{\ell})$ implies that there exists $p$ for which
        \[ \| \bm{x} - \bm{y}_p\|^2 - (\bm{f}_{\ell})_p + h_{\ell} < \| \bm{x} - \bm{y}_{\ell} \|^2 \le \| \bm{x} - \bm{y}_p \|^2 - (\bm{f}_{\ell})_p + g_{\ell}  \]
        and we thus note that any such $\bm{x}$ must lie in 
        \[ \|\bm{y}_p \|^2 - \|\bm{y}_{\ell} \|^2 - (\bm{f}_{\ell})_p + h_{\ell} < 2(\bm{y}_p - \bm{y}_{\ell})^\top \bm{x} \le  \| \bm{y}_p \|^2 - \|\bm{y}_{\ell} \|^2 - (\bm{f}_{\ell})_p + g_{\ell}, \]
        i.e. there exists $b_1, b_2 \in \mathbb{R}, \bm{a} \in \mathbb{R}^l$ for which $b_1 < \bm{a}^\top \bm{x} \le b_2$, where $b_2 - b_1 = g_{\ell} - h_{\ell}$ and $\bm{a} = 2(\bm{y}_p - \bm{y}_{\ell})$. This means that $\text{vol}((\mathbb{L}_{\ell}(\bm{f}_{\ell - 1}) \setminus \mathbb{L}_{\ell}(\bm{f}_{\ell})) \cap \mathcal{H})^2$ will be bounded above by
        \[ \text{vol}(\{ \bm{x} \in \mathcal{H}: b_1 \le \bm{a}^\top \bm{x} \le b_2 \})^2 \]
      %Let $\text{span}(\{ \bm{a} \})^\perp$ be the orthogonal complement of subspace $\text{span}(\{\bm{a}\}) \subseteq \mathcal{H}$ and define 
      Define $\mathcal{G}= \text{Proj}_{\text{span}\{\bm{a}\}^{\perp}}(\mathcal{H})$. We note that
      \[ \{ \bm{x} \in \mathcal{H}: b_1 \le \bm{a}^\top \bm{x} \le b_2 \} \subseteq \left \{ \bm{y} + \gamma \cdot \frac{\bm{a}}{\| \bm{a} \|^2}: \bm{y} \in \mathcal{G}, \gamma \in [b_1, b_2] \right \}   \]
        and thus we obtain
        \begin{align*}
            \text{vol}(\{ \bm{x} \in \mathcal{H}: b_1 \le \bm{a}^\top \bm{x} \le b_2 \})^2 \le \frac{\text{vol}(\mathcal{G})^2}{\| \bm{a} \|^2} \cdot (b_2 - b_1)^2.
        \end{align*} 
        Here, we note that as $\mathcal{G}$ is a projection of hyperrectangle $\mathcal{H} \in \mathbb{R}^l$ to some subspace of dimension $l - 1$, we obtain from Lemma \ref{projection} that
        \[ \text{vol}(\mathcal{G})^2 \le \frac{4l^2}{\xi^2} \text{vol}(\mathcal{H})^2 \le \frac{4l^2}{s^2} \text{vol}(\mathcal{H})^2, \]
        and thus we obtain
        \[ \text{vol}(\{ \bm{x} \in \mathcal{H}: b_1 \le \bm{a}^\top \bm{x} \le b_2 \})^2 \le \frac{4l^2}{s^4} \cdot (g_{\ell} - h_{\ell})^2 \cdot \text{vol}(\mathcal{H})^2 \]
       where by definition, $s \le \min_{1 \le i < j \le n} \| \bm{y}_i - \bm{y}_j \| \le \frac{1}{2} \| \bm{a} \|$.
        Similarly, when $j \not= \ell$, note that $\bm{x} \in \mathbb{L}_j(\bm{f}_{\ell}) \setminus \mathbb{L}_j(\bm{f}_{\ell - 1})$ implies that there exists $p$ for which 
        \[ \| \bm{x} - \bm{y}_p \|^2 - (\bm{f}_{\ell - 1})_p  < \| \bm{x} - \bm{y}_j \|^2 - (\bm{f}_{\ell - 1})_j \]
        while $\| \bm{x} - \bm{y}_{j} \|^2 - (\bm{f}_{\ell})_j \le \| \bm{x} - \bm{y}_p \|^2 - (\bm{f}_{\ell})_p$. This forces $p = \ell$, and thus we have
        \[ \| \bm{x} - \bm{y}_{\ell} \|^2 - g_{\ell} < \| \bm{x} - \bm{y}_j \|^2 - (\bm{f}_{\ell - 1})_j = \| \bm{x} - \bm{y}_j \|^2 - (\bm{f}_{\ell})_j\le \| \bm{x} - \bm{y}_{\ell} \|^2 - h_{\ell} \]
        and we thus note that any such $\bm{x}$ must lie in 
        \[ \| \bm{y}_{\ell} \|^2 - \|\bm{y}_j\|^2 + (\bm{f}_{\ell})_j - g_{\ell} < 2(\bm{y}_{\ell} - \bm{y}_j)^\top \bm{x} \le \| \bm{y}_{\ell} \|^2 - \|\bm{y}_j\|^2 + (\bm{f}_{\ell})_j - h_{\ell}, \]
        i.e. there exists $b_1, b_2 \in \mathbb{R}, \bm{a} \in \mathbb{R}^l$ for which $b_1 < \bm{a}^\top \bm{x} \le b_2$, where $b_2 -  b_1 = g_{\ell} - h_{\ell}$ and $\bm{a} = 2(\bm{y}_{\ell} - \bm{y}_j)$. This means that $\text{vol}((\mathbb{L}_j(\bm{f}_{\ell}) \setminus \mathbb{L}_j(\bm{f}_{\ell - 1})) \cap \mathcal{H})^2 $ will be bounded above by
        \[ \text{vol}(\{ \bm{x} \in \mathcal{H}: b_1 \le \bm{a}^\top \bm{x} \le b_2 \})^2 \]
        and by a similar argument, we can argue that this volume is bounded above by $\frac{4l^2}{s^4} \cdot (g_{\ell} - h_{\ell})^2 \cdot \text{vol}(\mathcal{H})^2$. 
        Summing the bounds for $1 \le j \le n$, we thus obtain 
        \[ \sum_{j = 1}^n (\text{vol}(\mathbb{L}_j(\bm{f}_{\ell - 1}) \cap \mathcal{H}) - \text{vol}(\mathbb{L}_j(\bm{f}_{\ell}) \cap \mathcal{H}))^2 \le \frac{4nl^2}{s^4} (g_{\ell} - h_{\ell})^2  \text{vol}(\mathcal{H})^2, \] 
        and thus taking $c_{\ell} = \frac{4nl^2}{ s^4}$ works.
        \vspace{4 mm} \newline 
To finish, we note that
        \begin{align*}
        \sum_{j = 1}^n (\text{vol}(\mathbb{L}_j(\bm{g}) \cap \mathcal{H}) - \text{vol}(\mathbb{L}_j(\bm{h}) \cap \mathcal{H}))^2   &= \sum_{j = 1}^n \left[ \sum_{\ell = 1}^{n} (\text{vol}(\mathbb{L}_j(\bm{f}_{\ell - 1}) \cap \mathcal{H}) - \text{vol}(\mathbb{L}_j(\bm{f}_{\ell}) \cap \mathcal{H})) \right]^2 \\ 
        &\le n \sum_{j = 1}^n \sum_{\ell = 1}^{n} (\text{vol}(\mathbb{L}_j(\bm{f}_{\ell - 1}) \cap \mathcal{H}) - \text{vol}(\mathbb{L}_j(\bm{f}_{\ell}) \cap \mathcal{H}))^2 \\ 
        &\le n \sum_{\ell = 1}^{n} \frac{4nl^2}{s^4} (g_{\ell} - h_{\ell})^2  \text{vol}(\mathcal{H})^2 \\ 
        &\le \frac{4n^2l^2}{s^4} \left( \sum_{\ell = 1}^n (g_{\ell} - h_{\ell})^2 \right)   \text{vol}(\mathcal{H})^2
        \end{align*}
        and thus we obtain the desired result.  
    \end{proof}

\begin{lemma}
    The function $\mathcal{E}$ is $L$-smooth where $L \in \text{poly} \left( n, l, k, \frac{1}{s} \right)$. 
\end{lemma}
\begin{proof}
    We will need to prove that for any $\bm{g}, \bm{h}\in \mathbb{R}^n$, we have
    \[ \| \nabla \mathcal{E}(\bm{g}) - \nabla \mathcal{E}(\bm{h}) \| \le L \| \bm{g} - \bm{h} \|\]
    for some $L \in \text{poly} \left( n, l,  k, \frac{1}{s} \right)$.  By Lemma \ref{hyperrec} applied to each of the $k$ hyperrectangles $\mathcal{H}_1, \mathcal{H}_2, \dots, \mathcal{H}_k$, we note that for any $1 \le \ell \le k$, we have
    \[\sum_{j = 1}^n (\text{vol}(\mathbb{L}_j(\bm{g}) \cap \mathcal{H}_{\ell}) - \text{vol}(\mathbb{L}_j(\bm{h}) \cap \mathcal{H}_{\ell}))^2 \le \frac{4n^2l^2}{s^4}  \left( \sum_{j = 1}^n |g_j - h_j|^2 \right) \text{vol}(\mathcal{H}_{\ell})^2 \]
    Finally, we note that  
    \begin{align*}
        \| \nabla \mathcal{E}(\bm{g}) - \nabla \mathcal{E}(\bm{h}) \|^2 &= \sum_{j = 1}^n  ( \nabla \mathcal{E}(\bm{g})_j - \nabla \mathcal{E}(\bm{h})_j)^2 \\ 
        &= \sum_{j = 1}^n \left( \sum_{\ell = 1}^k -\gamma_{\ell} \cdot (\text{vol}(\mathbb{L}_j(\bm{g}) \cap \mathcal{H}_{\ell}) - \text{vol}(\mathbb{L}_j(\bm{h}) \cap \mathcal{H}_{\ell})) \right)^2 \\ 
        &\le k \sum_{j = 1}^n \sum_{\ell = 1}^k \gamma_{\ell}^2 \cdot (\text{vol}(\mathbb{L}_j(\bm{g}) \cap \mathcal{H}_{\ell}) - \text{vol}(\mathbb{L}_j(\bm{h}) \cap \mathcal{H}_{\ell}))^2 \\ 
        &= k \sum_{\ell = 1}^k \gamma_{\ell}^2 \cdot \left[ \sum_{j = 1}^n (\text{vol}(\mathbb{L}_j(\bm{g}) \cap \mathcal{H}_{\ell}) - \text{vol}(\mathbb{L}_j(\bm{h}) \cap \mathcal{H}_{\ell}))^2 \right] \\ 
        &\le \frac{4n^2l^2 k}{s^4} \left( \sum_{j = 1}^n |g_j - h_j|^2 \right) \left( \sum_{\ell = 1}^k \gamma_{\ell}^2 \text{vol}(\mathcal{H}_{\ell})^2 \right) \\ 
        &\le \frac{4n^2l^2 k}{s^4} \| \bm{g} - \bm{h} \|^2 \left( \sum_{\ell = 1}^k \gamma_{\ell} \text{vol}(\mathcal{H}_{\ell}) \right)^2 = \frac{4n^2l^2 k}{s^4} \| \bm{g} - \bm{h} \|^2
    \end{align*}
    where here we use the fact that $1 = \int_{\mathbb{R}^l} d \alpha(\bm{x}) = \sum_{\ell = 1}^k \gamma_{\ell}\text{vol}(\mathcal{H}_{\ell})$, and thus by picking our constant $L$ to be
    \begin{equation} \label{eq-10}
    L := \frac{2nl k}{s^2} \in \text{poly} \left( n, l, k, \frac{1}{s} \right),
    \end{equation} 
    we have the desired result.
\end{proof} 

\begin{remark}
We will show in the next section (\ref{neces}) that dependence of the Lipschitz constant on the inverse minimum distance between any two points in the reference set is indeed necessary.
\end{remark}

With the lemmas established, we proceed to minimize 
%the energy function 
$-\mathcal{E}(\bm{g})$
using gradient descent. By leveraging a standard convergence guarantee, we can achieve $\varepsilon$-accuracy within $O \left( \frac{L}{\varepsilon} \right)$ time complexity.

\begin{theorem}[Corollary 2.1.2 from \cite{10.5555/2670022}]
    Let $f$ be convex and $L$-smooth on $\mathbb{R}^n$. Then, gradient descent with a step size $\eta = \frac{1}{L}$ ensures
    \[ f(\bm{x_k}) - f(\bm{x^{\ast}}) \le \frac{2L \| \bm{x_0 - x^{\ast}} \|^2}{k + 4} .\]
\end{theorem}

However, in our setting, we do not have access to exact gradients and instead rely on noisy gradient estimates obtained from our volume estimation algorithm, and thus we cannot rely on the classical convergence guarantee. Thus, we require a convergence guarantee that accounts for the noisy nature of the gradient estimates. %{\color{red}
Our use of inexact gradient descent here should be viewed as an application of the general inexact-gradient framework rather than a claim of a new generic optimization theorem. See, for example, d'Aspremont~\cite{dAspremont2008ApproximateGradient} and Vernimmen--Glineur~\cite{VernimmenGlineur2024InexactGradient} for examples of inexact gradient algorithms, although the assumptions in those works differ slightly from ours. The contribution needed for the present paper is to verify that the semidiscrete Wasserstein objective satisfies the hypotheses of such a framework with polynomially computable gradient estimates. For completeness, we therefore give a self-contained inexact-gradient argument tailored to the constants arising from our volume-estimation subroutine, rather than appeal to a black-box theorem for a different noise model.
%}

To this end, we consider an inexact gradient descent framework, where $f \equiv -\mathcal{E}$ is convex and $L$-smooth with $L = \frac{2nlk}{s^2}$. We will use $\widetilde{\nabla} f(\bm{x}_t)$ to stand for a noisy gradient in our algorithm satisfying
\begin{equation} 
\widetilde{\nabla} f(\bm{x}_t) = \nabla f(\bm{x}_t) + \bm{e}_t, 
\label{eq:tildegrad}
\end{equation}
where $\bm{e}_t$ represents the noise introduced in the $t$-th gradient estimation. The algorithm begins at $ \bm{g}_1 = \bm{0}$ and applies inexact gradient descent to the function $f \equiv -\mathcal{E}$ to maximize the energy $\mathcal{E}(\bm{g})$, aiming for the optimal value $\bm{g}^{\ast}$. This process generates a sequence of iterates $\bm{g}_1, \bm{g}_2, \dots$, targeting a final iterate $\bm{g}_{\overline{M}}$ after $\overline{M}$ iterations where $\overline{M}$ defined by:
%is defined in \eqref{iterate} and we see from the algorithm that $\overline{M} \le M$.
%Here, we will place the assumption that $\| \bm{e}_t \| \le \frac{\varepsilon'}{360nD^2}$. 
%{\color{red}
%For clarity, we define the iteration budget before the algorithm statement:
%}
\begin{align} \label{iterate}
    %{\color{red}
    M
    %} 
    &
    %{\color{red}
    := \frac{4}{\varepsilon'} \cdot 4800 n^2 D^4 L \quad \text{and} \quad  \overline{M} := \min \left( M, \min \left \{ t : \| \nabla f(\bm{g}_t) + \bm{e}_t \| \le \frac{\varepsilon'}{45nD^2} \right \} \right).
    %}
\end{align}

The value $\mathcal{E}(\bm{g}_{\overline{M}})$ will ultimately provide the accuracy required for our estimates of $\bm{\mu}^{\ast}$ and $\sigma^{\ast}$. To achieve this, we select
\begin{equation} \label{eq-11}
\varepsilon' = 2\varepsilon \cdot \frac{\left[ \left( \int_{\mathbb{R}^l} \| \bm{x} \|^2 d \alpha(\bm{x}) \right) - \left \| \int_{\mathbb{R}^l} \bm{x} \ d \alpha(\bm{x}) \right\|^2 \right]}{1 + \frac{1}{D} \left \| \int_{\mathbb{R}^l} \bm{x} \ d \alpha(\bm{x}) \right \| },
\end{equation} 
which we will prove to be bounded in terms of our parameters, and we will prove that upon termination of the following inexact gradient descent algorithm, the inequality $|\mathcal{E}(\bm{g}^{\ast}) - \mathcal{E}(\bm{g}_{\overline{M}})| \le \varepsilon'$ holds and use this to obtain the desired accuracy guarantee for $\bm{\mu}$ and $\sigma$. The algorithm is described in Algorithm 1.

%{\color{red}
At a high level, the algorithm repeatedly estimates the volumes of Laguerre-cell intersections to obtain an approximate gradient of the dual OT objective, takes an inexact gradient step, and stops once either the approximate gradient is small enough or the iteration budget is exhausted.
%}

\begin{algorithm} 
\caption{Inexact Gradient Descent ($f \equiv -\mathcal{E}$ here)} \label{inexact-grad}
\begin{algorithmic}
\State Initialize $\bm{g}_1= \bm{0}$, $t=1$.
\State Use volume estimation to compute $\widetilde{\nabla} f(\bm{g}_1)$ such that $\|\bm{e}_1\| \le \frac{\varepsilon'}{360n D^2} $.
\While{$\|\widetilde{\nabla} f(\bm{g}_t)\| > \frac{\varepsilon'}{45n D^2}$ %{\color{red}
and
%}
$t < M$}
    \State Update $\bm{g}_{t+1} = \bm{g}_t - \frac{1}{L} \widetilde{\nabla} f(\bm{g}_t)$.
    \State Compute $\widetilde{\nabla} f(\bm{g}_{t + 1})$ with $\|\bm{e}_{t + 1}\| \le \frac{\varepsilon'}{360n D^2}$ as in \eqref{eq:tildegrad}.
    \State Update $t=t+1$.
\EndWhile
\end{algorithmic}
\end{algorithm}

\begin{lemma}
    The value $\varepsilon'$ defined by \eqref{eq-11} is lower-bounded by $\frac{\varepsilon}{12}s^2$.
\end{lemma}

\begin{proof}
%{\color{red}
Write
\[ A:=\int_{\mathbb{R}^l} \| \bm{x} \|^2 \, d\alpha(\bm{x}) - \left\| \int_{\mathbb{R}^l} \bm{x}\, d\alpha(\bm{x}) \right\|^2 \]
for the numerator in \eqref{eq-11}, and
\[ B:=1+\frac{1}{D}\left\| \int_{\mathbb{R}^l} \bm{x}\, d\alpha(\bm{x}) \right\| \]
for the denominator. Then
\[ \varepsilon'=2\varepsilon \frac{A}{B}. \]

We first rewrite $A$ in centered form. Let
\[ \bm{c}:=\int_{\mathbb{R}^l}\bm{x}\,d\alpha(\bm{x}). \]
Since $\alpha$ is a probability measure, we have $\int_{\mathbb{R}^l} d\alpha(\bm{x})=1$. Expanding $\|\bm{x}-\bm{c}\|^2$, we obtain
\begin{align*}
\int_{\mathbb{R}^l}\|\bm{x}-\bm{c}\|^2\,d\alpha(\bm{x}) &= \int_{\mathbb{R}^l}\|\bm{x}\|^2\,d\alpha(\bm{x}) -2\bm{c}^{\top}\int_{\mathbb{R}^l}\bm{x}\,d\alpha(\bm{x}) +\|\bm{c}\|^2\int_{\mathbb{R}^l}d\alpha(\bm{x}) \\
&= \int_{\mathbb{R}^l}\|\bm{x}\|^2\,d\alpha(\bm{x}) -\left\|\int_{\mathbb{R}^l}\bm{x}\,d\alpha(\bm{x})\right\|^2.
\end{align*}
Hence
\[ A=\int_{\mathbb{R}^l}\|\bm{x}-\bm{c}\|^2\,d\alpha(\bm{x}). \]

We now lower-bound this centered second moment. For each hyperrectangle
\[ \mathcal{H}_i=[a_{i,1},b_{i,1}]\times\cdots\times[a_{i,l},b_{i,l}], \]
we have
\begin{align*}
\int_{\mathcal{H}_i}\|\bm{x}-\bm{c}\|^2\,d\bm{x} &= \sum_{j=1}^l \left( \int_{a_{i,j}}^{b_{i,j}} (t-c_j)^2\,dt \right) \prod_{h\neq j}(b_{i,h}-a_{i,h}).
\end{align*}
Thus it suffices to bound each one-dimensional integral from below. Fix an interval $[a,b]$ and a scalar $m\in\mathbb{R}$. Then
\[ \int_a^b (t-m)^2\,dt \]
is minimized when $m=(a+b)/2$, and its minimum value is
\[ \int_{-(b-a)/2}^{(b-a)/2} u^2\,du=\frac{(b-a)^3}{12}. \]
Therefore,
\[ \int_a^b (t-m)^2\,dt \ge \frac{(b-a)^3}{12} \qquad \text{for all } m\in\mathbb{R}. \]
Applying this with $a=a_{i,j}$, $b=b_{i,j}$, and $m=c_j$, we get
\begin{align*}
\int_{\mathcal{H}_i}\|\bm{x}-\bm{c}\|^2\,d\bm{x} &\ge \frac{1}{12}\sum_{j=1}^l (b_{i,j}-a_{i,j})^3 \prod_{h\neq j}(b_{i,h}-a_{i,h}) \\
&= \frac{1}{12}\,\mathrm{vol}(\mathcal{H}_i)\sum_{j=1}^l (b_{i,j}-a_{i,j})^2.
\end{align*}
By the definition of $s$ in Theorem~\ref{thm2-7}, every side length satisfies
\[ b_{i,j}-a_{i,j}\ge s, \]
so
\[ \int_{\mathcal{H}_i}\|\bm{x}-\bm{c}\|^2\,d\bm{x} \ge \frac{1}{12}\,\mathrm{vol}(\mathcal{H}_i)\,l s^2. \]
Multiplying by $\gamma_i$ and summing over $i$ yields
\begin{align*}
A = \int_{\mathbb{R}^l}\|\bm{x}-\bm{c}\|^2\,d\alpha(\bm{x}) &= \sum_{i=1}^k \gamma_i \int_{\mathcal{H}_i}\|\bm{x}-\bm{c}\|^2\,d\bm{x} \\
&\ge \frac{1}{12}ls^2 \sum_{i=1}^k \gamma_i \,\mathrm{vol}(\mathcal{H}_i) \\
&= \frac{1}{12}ls^2,
\end{align*}
since $\alpha$ is a probability distribution, so
\[ \sum_{i=1}^k \gamma_i \mathrm{vol}(\mathcal{H}_i)=1. \]

It remains to upper-bound $B$. Every point in the support of $\alpha$ lies in one of the hyperrectangles $\mathcal{H}_i$. Since the Euclidean norm is convex, its maximum over a hyperrectangle is attained at a corner, and all such corners belong to the reference set $\mathcal{X}$. By the definition of $D$, it follows that
\[ \|\bm{x}\|\le D \qquad \text{for every } \bm{x}\in \operatorname{supp}(\alpha). \]
Hence
\[ \left\|\int_{\mathbb{R}^l}\bm{x}\,d\alpha(\bm{x})\right\| \le \int_{\mathbb{R}^l}\|\bm{x}\|\,d\alpha(\bm{x}) \le D\int_{\mathbb{R}^l} d\alpha(\bm{x}) = D. \]
Therefore
\[ B = 1+\frac{1}{D}\left\|\int_{\mathbb{R}^l}\bm{x}\,d\alpha(\bm{x})\right\| \le 2. \]

Combining the bounds on $A$ and $B$, we obtain
\[ \varepsilon' = 2\varepsilon \frac{A}{B} \ge 2\varepsilon \cdot \frac{ls^2/12}{2} = \frac{\varepsilon l s^2}{12} \ge \frac{\varepsilon s^2}{12}, \]
which is the desired bound.
%}
\end{proof}

We begin by revisiting \eqref{eq-7} and observing that for any vector $\bm{g}$ such that $\sum_i g_i \neq 0$, we can define its centered counterpart as $\overline{\bm{g}} := \bm{g} - \left( \frac{1}{n} \sum_{i} g_i \right) \bm{1}$, which preserves the cost, i.e., $\mathcal{E}(\overline{\bm{g}}) = \mathcal{E}(\bm{g})$, while ensuring that $\sum_i \overline{g}_i = 0$. Consequently, any solution $\bm{g}$ can be mapped to a solution $\overline{\bm{g}}$ within the subspace 
\[ \mathcal{G}_0 := \left \{ \bm{g} \in \mathbb{R}^n : \sum_{i = 1}^n g_i = 0 \right \} \] 
without altering the cost. This observation motivates us to focus exclusively on solutions within the subspace $\mathcal{G}_0$.  
\vspace{3 mm} \newline 
To analyze this, we begin by proving the following claim about $\bm{g}$, which allows us to ensure the boundedness of $\bm{g}$ when restricted to the subspace $\mathcal{G}_0$.

\begin{lemma} If $\bm{g} \in \mathcal{G}_0$ is the optimizer of $\mathcal{E}$, then
\[ |g_i - g_j| \le 16nD^2, \quad \forall 1 \le i < j \le n.\]
In particular, under the restriction to the subspace $\mathcal{G}_0$, the optimal vector $\bm{g}$ is bounded by $16nD^2$ in the $\infty$-norm. 
\end{lemma}

\begin{proof}
Suppose $\mathbb{L}_i(\bm{g})$ and $\mathbb{L}_j(\bm{g})$ are distinct Laguerre cells, i.e., $i \neq j$. We first note that for any $j$, there exists an $\bm{x} \in \mathbb{L}_j(\bm{g})$ for which $\pi(\bm{x}, \bm{y}_j) > 0$. To show this, observe that $\int_{\mathbb{R}^l} d\pi(\bm{x}, \bm{y}_j) = \frac{1}{n} > 0$, and thus there exists $\bm{x}$ for which $d\alpha(\bm{x}) = \sum_{\ell = 1}^n d\pi(\bm{x}, \bm{y}_\ell) > 0$.  Therefore, by complementary slackness on $\pi$ and $\bm{g}, h$ and the fact that $d\pi^{\ast}(\bm{x}, \bm{y}_j) > 0$, we have
\[ \| \bm{x} - \bm{y}_j \|^2 - g_j^{\ast} - h^{\ast}(\bm{x}) = 0, \]
while dual feasibility gives us $\| \bm{x} - \bm{y}_{j'} \|^2 - g_{j'}^{\ast} - h^{\ast}(\bm{x}) \ge 0$ which gives us $\| \bm{x} - \bm{y}_{j'} \|^2 - g_{j'}^{\ast} \ge \| \bm{x} - \bm{y}_j \|^2 - g_j^{\ast}$, i.e. $\bm{x} \in \mathbb{L}_j(\bm{g})$. This shows that each Laguerre cell intersects with at least one hyperrectangle $\mathcal{H}_1, \dots, \mathcal{H}_k$, so let $\bm{z}_i \in \mathbb{L}_i(\bm{g}) \cap \bigcup_{p = 1}^k \mathcal{H}_p$ and $\bm{z}_j \in \mathbb{L}_j(\bm{g}) \cap \bigcup_{p = 1}^k \mathcal{H}_p$ and in particular, $\| \bm{z}_j \| \le D$. 

Consider the function $\bm{f}: [0,1] \to \mathbb{R}^l$ defined by $\bm{f}(\alpha) = (1 - \alpha) \bm{z}_i + \alpha \bm{z}_j$. Let $j,k$ be indices of two Laguerre cells that share a point. We denote $B_{j,k}$ denote the boundary between Laguerre cells $\mathbb{L}_j(\bm{g})$ and $\mathbb{L}_k(\bm{g})$. For any $\bm{x} \in B_{j,k}$, we have:
\[ \|\bm{x} - \bm{y}_j\|^2 - g_j = \|\bm{x} - \bm{y}_k\|^2 - g_k \implies g_j - g_k = \|\bm{x} - \bm{x}_j\|^2 - \|\bm{x} - \bm{y}_k\|^2. \]
Since $\bm{f}(0) = \bm{z}_i \in \mathbb{L}_i(\bm{g})$ and $\bm{f}(1) = \bm{z}_j \in \mathbb{L}_j(\bm{g})$, and Laguerre cells are convex,  we can find a sequence of indices $i = i_1, \dots, i_\ell = j$ and values $\alpha_1, \dots, \alpha_{\ell - 1} \in (0,1)$  with $2 \le \ell \le n$ such that $\bm{f}(\alpha_p) \in B_{i_p, i_{p + 1}}$ for all $1 \le p \le \ell - 1$.  
For each $1 \le p \le \ell - 1$, since $f(\alpha_p) \in B_{i_p, i_{p + 1}}$, it follows that:
\begin{align*}
    g_{i_p} - g_{i_{p + 1}} &= \|\bm{f}(\alpha_p) - \bm{y}_{i_p}\|^2 - \|\bm{f}(\alpha_p) - \bm{y}_{i_{p + 1}}\|^2 \\ 
    &\le \|\bm{f} (\alpha_p) - \bm{y}_{i_p} \|^2 \\
    &= \| (1 - \alpha_p)(\bm{z_i} - \bm{y}_{i_p}) + \alpha_p (\bm{z_j} - \bm{y}_{i_p}) \|^2 \\ 
    &\le 2 \left( \|\bm{z_i - y}_{i_p}\|^2 + \|\bm{z_j - y}_{i_p}\|^2 \right).
\end{align*}
Summing over all $1 \le p \le \ell - 1$, we obtain:
\[ g_i - g_j \le 2 \sum_{p=1}^{\ell - 1} \left( \|\bm{z_i} - \bm{y_{i_p}}\|^2 + \|\bm{z_j} - \bm{y_{i_p}} \|^2 \right) \le 2(\ell - 1) \cdot \left( \max_{1 \le p \le n} \|\bm{z_i - y_p}\|^2 + \max_{1 \le p \le n} \|\bm{z_j - y_p}\|^2 \right) \le 16nD^2. \] 
\end{proof}

Within $\mathcal{G}_0$, we have already established that $\bm{g}$ is bounded and thus by running gradient descent algorithm starting on $\bm{g}_0 = \bm{0}$, we hope to find a near-extremum value for $-\mathcal{E}$ inside the subspace $\mathcal{G}_0$.  Our next step is to prove that the gradient descent iterates are confined to a bounded set. 
%{\color{red}
This is precisely the setting of the paper, since the empirical measure places mass $\frac{1}{n}$ on each sample point.
%}

\begin{lemma}
Assume $\int_{\mathbb{R}^l} d\alpha(\bm{x})=1$ and each sink has mass $\frac{1}{n}$. Then the level set $S := \{ \bm{g} \in \mathcal{G}_0 \mid -\mathcal{E}(\bm{g}) \le -\mathcal{E}(\bm{0}) \}$ is entirely contained within the ball $B(\bm{0},20nD^2)$. 
\end{lemma}

\begin{proof}
Since 
%{\color{red}
$\frac{1}{n}\sum_{j=1}^n g_j=0$
%} 
for any $\bm{g} \in \mathcal{G}_0$, the energy function can be reduced to 
\[ \mathcal{E}(\bm{g}) = \sum_{j = 1}^n \int_{\mathbb{L}_j(\bm{g})} \left( \| \bm{x} - \bm{y}_j \|^2 - g_j \right) \, d\alpha(\bm{x}). \]
It follows immediately that $-\mathcal{E}(\bm{0}) \le 0$. 
We may reorder the components of $\bm{g}$ such that $g_1 \geq g_2 \geq \cdots \geq g_n$. 
We will first start by showing that the following summation
\[ \sum_{j = 1}^n \int_{\mathbb{L}_j(\bm{g})} \| \bm{x} - \bm{y}_j \|^2 \, d\alpha(\bm{x}) \]
is bounded above by a constant independent of $\bm{g}$. For any $\bm{x} \in \mathcal{H} := \bigcup_{i} \mathcal{H}_{i}$ and any $1 \le j \le n$, it holds that $\| \bm{x} - \bm{y}_j \|^2 \le 4D^2$. Thus, we have
\[ \sum_{j = 1}^n \int_{\mathbb{L}_j(\bm{g})} \| \bm{x} - \bm{y}_j \|^2 \, d\alpha(\bm{x}) \le 4D^2 \sum_{j = 1}^n \int_{\mathbb{L}_j(\bm{g})} 1 \, d\alpha(\bm{x}) = 4D^2 \cdot \sum_{\ell} \gamma_{\ell} \operatorname{vol}(\mathcal{H}_{\ell}) = 4D^2, \]
which is a constant independent of $\bm{g}$. 
\vspace{3 mm} \newline 
Next, observe that if $r > 1$ and $g_r < g_1 - 8D^2$, then the region $\mathbb{L}_r(\bm{g})$ does not meet a hyperrectangle. To be more precise, we see that if $g_1 - g_r > 8D^2$, then
\[ g_1 - g_r > \| \bm{x} - \bm{y}_1 \|^2 - \| \bm{x} - \bm{y}_r \|^2 \quad \forall \, \bm{x} \in \mathcal{H}. \]
This implies that no elements of the hyperrectangle can intersect with $\mathbb{L}_r(\bm{g})$, as desired.
We want to prove that the level set of $-\mathcal{E}(\bm{0})$ is entirely contained inside $B(\bm{0},20nD^2)$. To prove this, we just need to note that any $\bm{g}$ such that $\mathcal{E}(\bm{g}) \ge \mathcal{E}(\bm{0})$ must satisfy $g_1 \le 20nD^2$. We see that if $g_1 > 20nD^2$, then $g_n < 0$ and thus we have 
\begin{align*}
    0 &= g_1 + g_2 + \dots + g_n > 20nD^2 + n g_n 
\end{align*}
 and thus $g_n < -20D^2$, which implies $g_1 - g_n > 20(n + 1)D^2$. Thus we have 
\begin{align*} 
\mathcal{E}(\bm{g}) &=  \sum_{j = 1}^n \int_{\mathbb{L}_j(\bm{g})} \| \bm{x} - \bm{y}_j \|^2 \ d\alpha(\bm{x}) + \sum_{j = 1}^n \int_{\mathbb{L}_j(\bm{g})} -g_j \ d\alpha(\bm{x}) \\ 
&\le 4D^2 + \sum_{j = 1}^n \int_{\mathbb{L}_j(\bm{g})} - g_j \ d\alpha(\bm{x}).
\end{align*} 
Let $r$ be such that $g_r \ge g_1 - 8D^2$ and $g_{r + 1} < g_1 - 8D^2$. Such an $r$ must exist since $g_1 - g_n > 20(n + 1) D^2$ and $g_1 \ge g_1 - 8D^2$. Now, this implies that $\mathbb{L}_{r + 1}(\bm{g}) \cap \mathcal{H} = \dots = \mathbb{L}_n(\bm{g}) \cap \mathcal{H} = \varnothing$ and thus the above expression reduces to 
\begin{align*} 
4D^2 + \sum_{j = 1}^r \int_{\mathbb{L}_j(\bm{g})} - g_j \ d\alpha(\bm{x}) &\le 4D^2 + \sum_{j = 1}^r \int_{\mathbb{L}_j(\bm{g})} (8D^2 - g_1) \ d\alpha(\bm{x}) \\ 
&\le 4D^2 + \sum_{j = 1}^r \int_{\mathbb{L}_j(\bm{g})} (8D^2 - 20nD^2) \ d\alpha(\bm{x}) \\ 
&= 4D^2 + 8D^2 - 20nD^2 < 0 \le \mathcal{E}(\bm{0})
\end{align*}
\vspace{3 mm} \newline 
Therefore, this implies that if $\bm{g}$ is such that $\mathcal{E}(\bm{g}) \ge \mathcal{E}(\bm{0})$ must have $g_1 \le 20nD^2$. 
\end{proof}

In our inexact gradient analysis, where we set $f \equiv -\mathcal{E}$, the descent step follows a descent inequality as long as the gradient is not too small.

\begin{lemma}
    If the algorithm has not terminated at time $t$, then $\| \bm{e}_t \| < \frac{1}{7} \| \nabla f(\bm{g}_t) \|$. 
\end{lemma}
\begin{proof}
    Since the algorithm has not terminated on time $t$, we have $\| \widetilde{\nabla} f(\bm{g}_t) \| > \frac{\varepsilon'}{45n D^2}$. This implies that as this inexact gradient is computed with error $\| \bm{e}_t \| \le \frac{\varepsilon'}{360n D^2}$, we have 
    \[ \| \nabla f(\bm{g}_t) \| = \left \| \widetilde{\nabla} f(\bm{g}_t) - \bm{e}_t \right \| \ge \| \widetilde{\nabla} f(\bm{g}_t) \| - \| \bm{e}_t \| > \frac{7 \varepsilon'}{360n D^2} \ge 7 \| \bm{e}_t \|,  \]
 as desired. 
 \end{proof}

\begin{claim} \label{claim2-17}
For any $1 \le t \le \overline{M} - 1$, where $\overline{M}$ is defined as \eqref{iterate},
\[f(\bm{g}_{t+1}) \le f(\bm{g}_t) - \frac{1}{3L} \|\nabla f(\bm{g}_t)\|^2.\]
Here, $L$ is the smoothness constant in Equation \eqref{eq-10}.
\end{claim}

\begin{proof}
Using the smoothness of $f$ and the fact that $\| \bm{e}_t \| < \frac{1}{7} \| \nabla f(\bm{g}_t) \|$ since the algorithm has not terminated at time $t$, where $1 \le t \le \overline{M} - 1$, we have:
\begin{align*}
f(\bm{g}_{t+1}) &\le f(\bm{g}_t) + \langle \nabla f(\bm{g}_t), \bm{g}_{t+1} - \bm{g}_t \rangle + \frac{L}{2} \|\bm{g}_{t+1} - \bm{g}_t\|^2 \\
&= f(\bm{g}_t) - \left\langle \nabla f(\bm{g}_t), \frac{1}{L} (\nabla f(\bm{g}_t) + \bm{e}_t) \right\rangle + \frac{1}{2L} \|\nabla f(\bm{g}_t) + \bm{e}_t\|^2 \\
&= f(\bm{g}_t) - \frac{1}{2L} \|\nabla f(\bm{g}_t)\|^2 + \frac{1}{2L} \|\bm{e}_t\|^2 \\
&\le f(\bm{g}_t) - \frac{1}{3L} \|\nabla f(\bm{g}_t)\|^2.
\end{align*}
\end{proof}

This ensures that if the initial iterate $\bm{g}_1$ is in $S$, then all subsequent iterates $\bm{g}_t$ will remain within $S$, implying that we have $\| \bm{g}_t \| \le 20nD^2$ for all $t \ge 1$. 
%{\color{red}
Since Algorithm~\ref{inexact-grad} initializes $\bm{g}_1=\bm{0}$ and $S=\{\bm{g}\in\mathcal{G}_0:-\mathcal{E}(\bm{g})\le -\mathcal{E}(\bm{0})\}$ contains $\bm{0}$ by definition, the hypothesis $\bm{g}_1\in S$ is automatic.
%}
\vspace{3 mm} \newline 
We now derive an upper bound on the difference between the objective function evaluated at the current iterate and its optimal value. Given that the subsequence of iterates $\bm{g}_t$ satisfies $\| \bm{g}_t \| \leq 20nD^2$ for all $t \ge 1$ and thus $\| \bm{g}^{\ast} \| \le 20nD^2$, it follows from the subgradient inequality that
\begin{equation} \label{subgrad} f(\bm{g}_t) - f(\bm{g}^{\ast}) \leq \| \nabla f(\bm{g}_t) \| \cdot \| \bm{g}_t - \bm{g}^{\ast} \| \le 40nD^2 \cdot \| \nabla f(\bm{g}_t) \| , \quad \forall t \ge 1, \end{equation}  
where $\bm{g}^{\ast}$ denotes the maximizer of $\mathcal{E}$ (minimizer of $f$). 

To convert the descent inequality above into an explicit rate, we use the following standard recurrence bound; see, for example, \cite[p.~269]{BubeckLeeConvexOptimization}.

\begin{lemma} \label{lem:standard-quadratic-recurrence}
Let $p>0$, and let $\omega_0,\omega_1,\omega_2,\dots$ be nonnegative real numbers such that
\[ 0\le \omega_{k+1}\le \omega_k-p\omega_k^2 \qquad \text{for all } k\ge 0. \]
Then, for every $k\ge 0$,
\[ \omega_k\le \frac{1}{kp+\frac{1}{\omega_0}}. \]
\end{lemma}

\begin{lemma} \label{error-bound}
For all $2 \le t \le \overline{M}$, where $\overline{M}$ is defined in \eqref{iterate}, the following inequality holds:
\[ f(\bm{g}_t) - f(\bm{g}^{\ast}) \le \frac{4}{t}\cdot 4800 n^2 D^4 L. \]
\end{lemma}

\begin{proof}
For convenience, define $a_t:=f(\bm{g}_t)-f(\bm{g}^{\ast}), 1\le t\le \overline{M}$. By Claim~\ref{claim2-17}, for every $1\le t\le \overline{M}-1$ we have
\[ a_{t+1}\le a_t-\frac{1}{3L}\|\nabla f(\bm{g}_t)\|^2. \]
On the other hand, by the subgradient inequality \eqref{subgrad},
\[ a_t\le 40nD^2\,\|\nabla f(\bm{g}_t)\|, \]
and hence
\[ \|\nabla f(\bm{g}_t)\|^2\ge \frac{a_t^2}{1600\,n^2D^4}. \]
Substituting this into the previous inequality gives
\[ a_{t+1}\le a_t-\frac{1}{4800\,n^2D^4L}\,a_t^2, \qquad 1\le t\le \overline{M}-1. \]
Thus, if we set
\[ p:=\frac{1}{4800\,n^2D^4L}, \]
then
\[ a_{t+1}\le a_t-p\,a_t^2 \qquad \text{for all } 1\le t\le \overline{M}-1. \]
We next bound the initial error $a_1$. Using again \eqref{subgrad}, together with $\bm{g}_1=\bm{0}$ and $\nabla f(\bm{g}^{\ast})=\bm{0}$, we obtain
\begin{align*}
a_1 &\le 40nD^2\,\|\nabla f(\bm{g}_1)\| \\
&=40nD^2\,\|\nabla f(\bm{g}_1)-\nabla f(\bm{g}^{\ast})\| \\
&\le 40nD^2\,L\|\bm{g}_1-\bm{g}^{\ast}\| \\
&\le 800\,n^2D^4L.
\end{align*}
Now define a new sequence by
\[ \omega_k:=a_{k+1}, \qquad k=0,1,\dots,\overline{M}-1. \]
Then $\omega_0=a_1$, and for every $0\le k\le \overline{M}-2$,
\[ 0\le \omega_{k+1}\le \omega_k-p\,\omega_k^2. \]
Therefore, by Lemma~\ref{lem:standard-quadratic-recurrence},
\[ \omega_k\le \frac{1}{kp+\frac{1}{\omega_0}} \qquad \text{for all } 0\le k\le \overline{M}-1. \]
Equivalently, for every $1\le t\le \overline{M}$,
\[ a_t\le \frac{1}{(t-1)p+\frac{1}{a_1}}. \]
Since
\[ a_1\le 800\,n^2D^4L, \]
we have
\[ \frac{1}{a_1}\ge \frac{1}{800\,n^2D^4L}=6p. \]
Hence
\[ a_t\le \frac{1}{(t-1)p+6p} =\frac{1}{(t+5)p} =\frac{4800\,n^2D^4L}{t+5}. \]
In particular, for every $2\le t\le \overline{M}$,
\[ a_t\le \frac{4800\,n^2D^4L}{t+5} \le \frac{4}{t}\cdot 4800\,n^2D^4L. \]
This proves the claim.
\end{proof}

\begin{lemma} \label{termination-error}
If Algorithm \ref{inexact-grad} terminates at time $\overline{M}$, then $f(\bm{g}_{\overline{M}}) - f(\bm{g}^{\ast}) \le \varepsilon'$. 
\end{lemma} 

\begin{proof}
We note that there are two possible conditions of termination. For the first case, in which case we have $\| \widetilde{\nabla} f(\bm{g}_{\overline{M}}) \| < \frac{\epsilon'}{45n D^2} $, we obtain
\[ \| \nabla f(\bm{g}_{\overline{M}}) \| \le \| \bm{e}_M \| + \| \widetilde{\nabla} f(\bm{g}_{\overline{M}})  \| \le \frac{\varepsilon'}{40n D^2} . \]
In conclusion, by \eqref{subgrad} we have  
\begin{align*}
    f(\bm{g}_{\overline{M}}) - f(\bm{g}^{\ast}) &\le 40nD^2\| \nabla f(\bm{g}_{\overline{M}}) \| \le \varepsilon' ,
\end{align*}
as desired. Otherwise, for the second case of termination when $\overline{M} = M$, we note that from \eqref{iterate} and Lemma~\ref{error-bound},
\[ f(\bm{g}_M) - f(\bm{g}^{\ast}) \le \frac{4}{M} \cdot 4800n^2 D^4 L \le \varepsilon' \]
\end{proof}

\begin{claim}
    The running time of Algorithm \ref{inexact-grad} is polynomial in $\text{poly} \left( n, l, k, D, \frac{1}{s}, \frac{1}{\varepsilon}, \log \frac{1}{\eta} \right)$.
\end{claim}
\begin{proof}
    By the design of the algorithm, this algorithm runs at most 
    \[ M = \frac{4}{\varepsilon'} \cdot 4800n^2 D^4 L \]
    time. At each of the algorithm iteration, it must be the case that $\| \widetilde{\nabla} f(\bm{g}_t) \| > \frac{ \varepsilon'}{45n D^2}$. We note that $\nabla \mathcal{E}(\bm{g}_t)$ can be computed as
    \[ \nabla \mathcal{E}(\bm{g}_t)_j = \sum_{\ell = 1}^k -\gamma_{\ell} \cdot \text{vol}(\mathbb{L}_j(\bm{g}_t) \cap \mathcal{H}_{\ell}) + \frac{1}{n}, \quad 1 \le j \le n. \]
    We will now use our volume computation algorithm to obtain an estimate $v_{t,h}$ for $\text{vol}(\mathbb{L}_j(\bm{g}_t) \cap \mathcal{H}_h)$ within $\frac{\overline{\varepsilon}}{\sqrt{n}}$ accuracy, where $\overline{\varepsilon} := \frac{\varepsilon'}{360nD^2}$ for each $1 \le t \le \overline{M}$ and $1 \le h \le k$, i.e.,
    \[ |v_{t,h} - \text{vol}(\mathbb{L}_j(\bm{g}_t) \cap \mathcal{H}_h)| \le \frac{\overline{\varepsilon}}{\sqrt{n}} \cdot \text{vol}(\mathbb{L}_j(\bm{g}_t) \cap \mathcal{H}_h) \le \frac{\overline{\varepsilon}}{\sqrt{n}} \cdot \text{vol}(\mathcal{H}_h)   \] 
    and let us consider the noisy gradient estimate $v_{t,1}, \dots, v_{t,k}$ as follow: 
    \begin{align*}
        \widetilde{\nabla} \mathcal{E}(\bm{g}_t)_j = \sum_{\ell = 1}^k -\gamma_{\ell} \cdot v_{t,\ell} + \frac{1}{n}, \quad 1 \le j \le n.
    \end{align*}
    Therefore, for this choice of $\widetilde{\nabla} \mathcal{E}(\bm{g}_t)$, we must have 
  \begin{align*}
    \| \widetilde{\nabla} \mathcal{E}(\bm{g}_t) - \nabla \mathcal{E}(\bm{g}_t) \|^2 &= \sum_{j =1}^n |\widetilde{\nabla} \mathcal{E}(\bm{g}_t)_j - \nabla \mathcal{E}(\bm{g}_t)_j |^2 \\ 
    &= \sum_{j = 1}^n \left| \sum_{\ell = 1}^k \gamma_{\ell}  (v_{t,\ell} - \text{vol}(\mathbb{L}_j(\bm{g}_t) \cap \mathcal{H}_{\ell})) \right|^2 \\ 
    &\le \sum_{j = 1}^n \left( \sum_{\ell = 1}^k \gamma_{\ell} \cdot |v_{t,\ell} - \text{vol}(\mathbb{L}_j(\bm{g}_t) \cap \mathcal{H}_{\ell})| \right)^2 \\ 
    &\le \sum_{j = 1}^n \left( \sum_{\ell = 1}^k \gamma_{\ell} \cdot \frac{\overline{\varepsilon}}{\sqrt{n}} \cdot \text{vol}(\mathcal{H}_{\ell}) \right)^2 \\ 
    &= \frac{(\overline{\varepsilon})^2}{n} \sum_{j = 1}^n \left( \sum_{\ell = 1}^k \gamma_{\ell} \text{vol}(\mathcal{H}_{\ell}) \right)^2 \\ 
    &\le (\overline{\varepsilon})^2 
\end{align*}
and thus we have $\| \widetilde{\nabla} \mathcal{E}(\bm{g}_t) - \nabla \mathcal{E}(\bm{g}_t) \| \le \overline{\varepsilon}$. By our volume estimation oracle, we can obtain an $\overline{\varepsilon}$-estimate of $\nabla \mathcal{E}(\bm{g}_t)$ using $k$ calls to volume estimation. This process requires $O \left( \text{poly} \left( n, \frac{1}{\overline{\varepsilon}}, \log \frac{1}{\eta'} \right) \right)$ separation oracle calls, where $\eta' = \frac{\eta}{kM}$. Applying the union bound, the probability that all $kM$ separation oracles yield correct volume estimations is at least
\begin{align*}
    &\mathbb{P} \left[ \text{All } kM \text{ oracles yield correct volume estimation} \right] \\ 
    &= 1 - \mathbb{P}\left[\text{At least one oracle yields incorrect volume estimation}\right] \\ 
    &\geq 1 - kM \cdot \eta' = 1 - \eta.
\end{align*}
Thus, the desired probability bound is achieved.
\vspace{3 mm} \newline 
Additionally, each separation oracle call operates in $O(nl)$ time. Performing this procedure for each $1 \leq i \leq M$ results in a total of
\[ \sum_{i = 1}^M O \left( \text{poly} \left[ n, l,  \frac{1}{\overline{\varepsilon}}, \log \frac{1}{\eta'} \right] \right) = O \left( \text{poly} \left[ n, k, M,  \frac{1}{\overline{\varepsilon}}, \log \frac{1}{\eta'} \right] \right) \] 
running time. Here, we notice that we have
\[ L = \frac{2nlk}{s^2} \quad \text{and} \quad M = \frac{4}{\varepsilon'} \cdot 4800n^2 D^4 L \in O \left( \text{poly} \left[  n,  l, k, D, \frac{1}{\varepsilon},  \frac{1}{s} \right] \right) . \]
Finally, we also note that $\log \frac{1}{\eta'} = \log \frac{kM}{\eta} \in O \left( kM \log \frac{1}{\eta} \right) = O \left( \text{poly} \left[ n, l, k,  D, \frac{1}{\varepsilon},  \frac{1}{s}, \log \frac{1}{\eta} \right]  \right)$. 
This gives us $O \left( \text{poly} \left[ n, l,  k, D, \frac{1}{\varepsilon},  \frac{1}{s}, \log \frac{1}{\eta} \right]  \right)$ time to obtain $\bm{g}_{\overline{M}}$. 
\vspace{3 mm} \newline
 We recall that $p^{\ast} = \mathcal{E}(\bm{g}^{\ast})$ gives us the optimal solution to the dual program, and thus giving us an estimate $\rho := \frac{1}{2} \left[  \mathcal{E}(\bm{g}_{\overline{M}}) - \int_{\mathbb{R}^l} \| \bm{x} \|^2 d \alpha(\bm{x}) - \frac{1}{n}\sum_{j = 1}^n \| \bm{y}_j \|^2 \right]$ which can be computed exactly, for which we have
\[2 \left| \int_{\mathbb{R}^l} \sum_{j = 1}^n \bm{x}^\top \bm{y}_j \ d\pi(\bm{x}, \bm{y}_j) - \rho \right| = |\mathcal{E}(\bm{g}_{\overline{M}}) - \mathcal{E}(\bm{g}^{\ast})| \le \varepsilon' \]
from Lemma \ref{termination-error} and thus by our choice of $\varepsilon'$, we have 
\[ \left| \int_{\mathbb{R}^l} \sum_{j = 1}^n \bm{x}^\top \bm{y}_j \ d\pi(\bm{x}, \bm{y}_j) - \rho \right| \le \varepsilon \cdot \frac{\left[ \left( \int_{\mathbb{R}^l} \| \bm{x} \|^2 d \alpha(\bm{x}) \right) - \left \| \int_{\mathbb{R}^l} \bm{x} \  d \alpha(\bm{x}) \right\|^2 \right]}{1 + \frac{1}{D} \left \| \int_{\mathbb{R}^l} \bm{x} \ d \alpha(\bm{x}) \right \| }\]
Therefore, by considering
  \begin{align*}
    \hat{\sigma} &= \frac{\rho - \left( \frac{1}{n}\sum_{j=1}^n \bm{y}_j \right)^\top \left( \int_{\mathbb{R}^l} \bm{x} \, d\alpha(\bm{x}) \right)}{\left( \int_{\mathbb{R}^l} \| \bm{x} \|^2 \, d\alpha(\bm{x}) \right) - \left\| \int_{\mathbb{R}^l} \bm{x} \, d\alpha(\bm{x}) \right\|^2}\\ 
    \hat{\bm{\mu}} &= \hat{\sigma} \int_{\mathbb{R}^l} \bm{x} \, d\alpha(\bm{x}) - \frac{1}{n}\sum_{j=1}^n \bm{y}_j,
    \end{align*}
    we thus obtain that
    \begin{align*}
        |\hat{\sigma} - \sigma^{\ast}| &= \frac{1}{\left( \int_{\mathbb{R}^l} \| \bm{x} \|^2 \, d\alpha(\bm{x}) \right) - \left\| \int_{\mathbb{R}^l} \bm{x} \, d\alpha(\bm{x}) \right\|^2}  \left| \int_{\mathbb{R}^l} \sum_{j = 1}^n \bm{x}^\top \bm{y}_j \ d\pi(\bm{x}, \bm{y}_j) - \rho \right| \\ 
        &\le \varepsilon \\ 
        \| \hat{\bm{\mu}} - \bm{\mu}^{\ast} \| &= |\hat{\sigma} - \sigma^{\ast}| \left \| \int_{\mathbb{R}^l} \bm{x} \ d\alpha(\bm{x})\right \| \le \varepsilon D, 
    \end{align*}
    as desired.
\end{proof}
\end{proof}

\subsection{Necessity of Parameters in Smoothness Constants} \label{neces}
In this section, we will show that the established Lipschitz constant necessarily depends on the inverse of both the closest separation between two sample points and the minimum width of the hyperrectangles in the source distribution, as follow: 
\vspace{3 mm} \newline 
\textbf{Minimum separation between two sample points}: Fix $m \in \mathbb{N}$. Consider the following setup where we have $l = 1, k = 1$, $\bm{g_{m}} = \begin{pmatrix} 0 \\ 0 \end{pmatrix} , \bm{g'_m} = \begin{pmatrix} 0 \\ \frac{1}{m} \end{pmatrix}$ and define $y_{1,m}, y_{2,m} \in \mathbb{R}$, where $y_{1,m} =  -\frac{1}{m}, y_{2,m} = \frac{1}{m}$ where the corresponding hyperrectangle $\mathcal{H}$ in the source distribution is $[-1,1]$ with weight $\gamma = \frac{1}{2}$. We note that for any fixed $m \ge 1$, we can easily verify that the Laguerre cells are
    \begin{align*}
        \mathbb{L}_1(\bm{g_m}) &= \left \{ x \in \mathbb{R}: x \le 0 \right  \} \\ 
        \mathbb{L}_2(\bm{g_m}) &= \{ x \in \mathbb{R}: x \ge 0 \} \\
        \mathbb{L}_1(\bm{g'_m}) &= \left \{ x \in \mathbb{R}: x \le -\frac{1}{4} \right \}  \\ 
        \mathbb{L}_2(\bm{g'_m}) &= \left \{ x \in \mathbb{R}: x \ge -\frac{1}{4} \right \} 
    \end{align*}
    and thus $\| \nabla \mathcal{E}(\bm{g_m}) - \nabla \mathcal{E}(\bm{g_m'}) \| \not= 0$ is independent of $m$ for any $m \ge 1$, while $\| \bm{g_m - g_m'}\| \to 0$ when $m \to \infty$. This shows that when $m = \frac{2}{| y_{1,m} - y_{2,m} |} \to \infty$, $\frac{\| \nabla \mathcal{E}(\bm{g_m}) - \nabla \mathcal{E}(\bm{g'_m}) \| }{\| \bm{g_m - g'_m} \|} \to \infty$ as well, and thus this shows that $L \to \infty$ when $m \to \infty$. 
    \vspace{3 mm} \newline 
\textbf{Minimum width of the hyperrectangles in the source distribution}: Fix $m \in \mathbb{N}$. Let $l = 2, k = 1, \bm{g_m} = \begin{pmatrix} 0 \\ 0 \end{pmatrix}, \bm{g'_m} = \begin{pmatrix} 0 \\ \frac{1}{m} \end{pmatrix}$ and define $\bm{y}_1, \bm{y}_2 \in \mathbb{R}^2$ such that $\bm{y}_1 = \begin{pmatrix} -1 \\ 0 \end{pmatrix} , \bm{y}_2 = \begin{pmatrix} 1 \\ 0 \end{pmatrix}$ where the corresponding hyperrectangle $\mathcal{H}$ in the source distribution is $\left[ -\frac{1}{m}, 0 \right] \times \left[ 0, m \right]$ with weight $\gamma = 1$. For any fixed $m \ge 1$, we can easily verify that the Laguerre cells are 
    \begin{align*}
        \mathbb{L}_1(\bm{g_m}) &= \{ \bm{x} \in \mathbb{R}^2: x_1 \le 0 \} \\ 
        \mathbb{L}_2(\bm{g_m}) &= \{ \bm{x} \in \mathbb{R}^2: x_1 \ge 0 \} \\ 
        \mathbb{L}_1(\bm{g_m'}) &= \left \{ \bm{x} \in \mathbb{R}^2: x_1 \le -\frac{1}{4m} \right \} \\
        \mathbb{L}_2(\bm{g_m'}) &= \left \{ \bm{x} \in \mathbb{R}^2: x_1 \ge -\frac{1}{4m} \right \} 
    \end{align*}
    Therefore, for any $m \ge 1$, we have 
    \begin{align*}
        \| \nabla \mathcal{E}(\bm{g_m}) - \nabla \mathcal{E}(\bm{g_m'}) \|^2 = \sum_{j = 1}^2 \left( \text{vol}(\mathbb{L}_j(\bm{g_m}) \cap \mathcal{H}) - \text{vol}(\mathbb{L}_j(\bm{g'_m}) \cap \mathcal{H}) \right)^2 = \frac{1}{8}
    \end{align*}
    independent of $m$, while $\| \bm{g_m - g'_m} \| \to 0$ when $m \to \infty$. This shows that when $m = \frac{1}{\xi} \to \infty$, where $\xi$ is the minimum width of the hyperrectangle, then we also have $\frac{\| \nabla \mathcal{E}(\bm{g_m}) - \nabla \mathcal{E}(\bm{g'_m}) \|}{\| \bm{g_m - g'_m} \|} \to \infty$ as well, and thus this shows that $L \to \infty$ when $m \to \infty$. 
\renewcommand{\ell}{s}
\section{Hardness of Finding a Single Optimal Center in MLE} \label{section 3}
In this section, we present our main hardness result, highlighting a key distinction between the Wasserstein distance and Maximum Likelihood Estimation (MLE). Specifically, we prove that the problem of finding a single optimal center is NP-hard for MLE, whereas it is solvable in polynomial time for the Wasserstein metric as shown in the previous section. 
%{\color{red}
This comparison is again complexity-theoretic: the point is not that MLE is universally inferior as a statistical procedure, but that for this finitely specified family the exact optimization landscapes of the two objectives have sharply different computational behavior.
%}
\vspace{3 mm} \newline 
\textbf{Problem Setup.} Given an unknown parameter $\bm{\theta}$ of some distribution $D \equiv \Omega(\bm{y}; \bm{\theta})$ with its associated probability density function $f(\bm{y}) = f_0(\bm{y} - \bm{\theta})$ for some probability density function $f_0$, and let $\bm{y}_1, \bm{y}_2, \dots, \bm{y}_n \in \mathbb{R}^l$ be $n$ sample points from $D$. We would like to estimate $\bm{\theta}$ using maximum likelihood, i.e. we choose $\bm{\theta} := \bm{\theta}^{\ast}$ such that
\[ \bm{\theta}^{\ast} := \argmax_{\bm{\theta} \in \mathbb{R}^l} \prod_{i = 1}^n f(\bm{y}_i; \bm{\theta}) =\argmax_{\bm{\theta} \in \mathbb{R}^l} \prod_{i = 1}^n f_0(\bm{y}_i - \bm{\theta})  \]
Our main result is as follows. 

\begin{theorem}
Given the problem setup in the case that $f_0$ is a piecewise constant function on a finite union of hyperrectangles, it is NP-hard to determine whether a $\bm{\theta}$ exists for which the likelihood is nonzero.
\end{theorem}
\begin{proof}
The construction is based on the 3-SAT problem, a well-known NP-hard problem. Consider $l$ variables denoted as $a_1, \dots, a_l$, and consider a Boolean formula $\mathcal{A}_1 \wedge \mathcal{A}_2 \wedge \dots \wedge \mathcal{A}_n$ with $n$ clauses, each having exactly three literals. We may without loss of generality assume that no clause contains the same variable more than once. We also assume without loss of generality that every variable appears in at least one clause. Now, we proceed to create an MLE problem based on the given Boolean formula. 
\vspace{1 mm} \newline 
Consider $\bm{y}_1, \dots, \bm{y}_n \in \mathbb{R}^l$ such that $\bm{y}_\ell = (y_{\ell,j})_{1 \le j \le l}$ and
    \[ y_{\ell,j} = \begin{cases} 0, &\text{if variable }a_j \ \text{does not occur in } \mathcal{A}_\ell, \\ \ell, &\text{if variable }a_j \ \text{occurs in } \mathcal{A}_\ell. \end{cases} \]
    
For any $1 \le \ell \le n$, it is assumed that $\mathcal{A}_\ell$ contains three variables $a_p, a_q, a_r$ with $1 \le p < q < r \le l$. In this scenario, we observe that there are $7$ possible true-false assignments that satisfy the clause $\mathcal{A}_\ell$, as $\mathcal{A}_{\ell}$ is a disjunction, i.e. there is only $1$ possible assignment for which this is not satisfied. For each of these $7$ possible true-false assignment to $(a_p, a_q, a_r)$, which we label as $\mathcal{X}_1, \dots, \mathcal{X}_7$, we define $D_{\ell}, D_{\ell,1} , \dots, D_{\ell,7} \subseteq \mathbb{R}^l$ such that 
    \[ D_{\ell} = \bigcup_{1 \le i \le 7} D_{\ell,i}\]
    where for any $1 \le i \le 7$, we define $D_{\ell,i} = I_{1,\ell,i} \times I_{2,\ell,i} \times \dots \times I_{l,\ell,i}$ in which for any $1 \le j \le l$, we have
    \[ I_{j,\ell,i} = \begin{cases} [0,0.5], &\text{if } a_j \ \text{does not occur in } \mathcal{A}_{\ell}, \\ [\ell + 0.5 - \varepsilon, \ell + 0.5 + \varepsilon], &\text{if } a_j \ \text{occurs in } \mathcal{A}_{\ell}  \text{ and }  a_j \ \text{is true in } \mathcal{X}_i, \\ [\ell  - \varepsilon, \ell + \varepsilon], &\text{if } a_j \ \text{occurs in } \mathcal{A}_{\ell}  \text{ and }  a_j \ \text{is false in } \mathcal{X}_i.\end{cases} \]
    where $\varepsilon = \frac{1}{80}$. We can then define 
    \[ \mathcal{D} = \bigcup_{1 \le \ell \le n} D_{\ell}. \] 
    We will first observe that all $(D_{\ell,i})_{1 \le \ell \le n, 1 \le i \le 7}$ are disjoint, i.e. $D_{\ell, i}$ is disjoint from $D_{\ell', i'}$. To see this, we note that for any $x, y \in \mathbb{N}$, the three intervals $[0,0.5], [x + 0.5 - \varepsilon, x + 0.5 + \varepsilon]$ and $[y - \varepsilon, y + \varepsilon]$ are all pairwise disjoint as $\varepsilon$ is chosen to be small enough. This implies that any point $\bm{y} \in \mathcal{D}$, coordinate-wise there is at most one interval that contains this point, and thus uniquely determines which interval contains this point. 
    \vspace{3 mm} \newline 
    Now, we will define a probability density function $f_0$ based on distribution $\mathcal{D}$. Here, $f_0$ is given by
    \[ f_0(\bm{y}) = \begin{cases} \gamma &\text{if } \bm{y} \in \mathcal{D} \\ 0 &\text{otherwise} \end{cases}  \]
    where $\gamma \in \mathbb{R}$ is defined to be such that $\gamma = \left( \int_\mathcal{D} 1 \right)^{-1}$, which can be explicitly calculated as  
    \[ \gamma = (7n(0.5)^{l - 3} (2 \varepsilon)^3)^{-1} > 0 \]
    since we have established that all $(D_{\ell,i})_{1 \le \ell \le n, 1 \le i \le 7}$ are disjoint.
    \vspace{3 mm} \newline 
    Our proposed MLE problem is hence as follow: 
    \vspace{2 mm} \newline 
    \textbf{Proposed MLE Problem.} Given an unknown parameter $\bm{\theta}$ of distribution $\mathcal{D}$, which we have defined above, with some appropriate probability density function $f$ which satisfies $f(\bm{y}) = f_0(\bm{y} - \bm{\theta})$, where $f_0$ is defined above. The given $n$ sample points are $\bm{y}_1, \dots, \bm{y}_n \in \mathbb{R}^l$, we wish to estimate $\bm{\theta}$ using maximum likelihood estimation (MLE). 
    \vspace{3 mm} \newline 
    The main theorem is proved by the following claim.

    \begin{claim}
        There exists $\hat{\bm{\theta}} \in \mathbb{R}^l$ such that $f_0(\bm{y}_\ell - \hat{\bm{\theta}}) > 0$ for all $1 \le \ell \le n$ if and only if there exists a satisfying assignment to the original Boolean formula. 
    \end{claim}
    \begin{proof} 
     First, suppose there exists a satisfying assignment to the original Boolean formula. We will prove that there exists $\hat{\bm{\theta}} \in \mathbb{R}^l$ such that $f_0(\bm{y}_\ell - \hat{\bm{\theta}}) > 0$ for all $1 \le \ell \le n$.
    \vspace{3 mm} \newline 
    Consider a satisfying assignment $\mathcal{X}$ of $(a_1, \dots, a_l)$ to the original Boolean formula and construct $\overline{\bm{\theta}} = (\overline{\theta}_j)_{1 \le j \le l} \in \mathbb{R}^l$, where 
        \[ \overline{\theta}_j = \begin{cases}  -0.5, &\text{if } a_j \ \text{is true in } \mathcal{X}, \\ 
        0, &\text{if }  a_j \ \text{is false in } \mathcal{X}. \end{cases}  \] 
    We will now prove that $f_0(\bm{y}_\ell - \overline{\bm{\theta}}) > 0$ for all $1 \le \ell \le n$. Fix such an $\ell$.We will prove that $f_0(\bm{y}_{\ell} - \overline{\bm{\theta}}) > 0$, i.e. $\bm{y}_{\ell} - \overline{\bm{\theta}} \in \mathcal{D}$, by our definition of $f_0$. 
    By definition of $\bm{y}_{\ell} = (y_{\ell,j})_{1 \le j \le l}$, we would therefore have 
    \[ y_{\ell,j} - \overline{\theta}_j = \begin{cases} - \overline{\theta}_j, &\text{if variable } a_j \ \text{does not occur in } \mathcal{A}_{\ell}, \\ \ell - \overline{\theta}_j, &\text{if variable } a_j \ \text{occurs in } \mathcal{A}_{\ell}, \end{cases} \]
    for any $1 \le j \le l$. We'll consider two cases: If $a_p$ does not occur in $\mathcal{A}_{\ell}$, then $y_{\ell,p} - \overline{\theta}_p = -\overline{\theta}_p \in [0,0.5] = I_{p,\ell,i}$ for all $1 \le i \le 7$, corresponding to the $7$ possible satisfying assignment for clause $\mathcal{A}_{\ell}$. Otherwise, if $a_p$ occurs in $\mathcal{A}_{\ell}$, we have 
    \[ y_{\ell ,p} - \overline{\theta}_p = \ell - \overline{\theta}_p = \begin{cases} \ell, &\text{if } a_p \ \text{is false in } \mathcal{X}, \\ \ell + 0.5, &\text{if } a_p \ \text{is true in } \mathcal{X} .\end{cases} \] However, we see that
    \[ I_{p,\ell,i} = \begin{cases} [\ell - \varepsilon, \ell + \varepsilon] &\text{if } a_p \ \text{is false in } \mathcal{X}_i \\ [\ell +0.5 - \varepsilon, \ell +0.5 + \varepsilon] &\text{if } a_p \ \text{is true in } \mathcal{X}_i \end{cases} \]
    for some fixed $1 \le i \le 7$. As $\mathcal{X}$ is a satisfying assignment to the original Boolean formula, then it must be a satisfying assignment to clause $\mathcal{A}_{\ell}$ as well, and thus must be equal to $\mathcal{X}_i$ for some $1 \le i \le 7$ by our assumption. In either of the cases, we obtain that $\ell - \overline{\theta}_p \in I_{p,\ell,i}$. This therefore gives us that $f_0(\bm{y}_\ell - \overline{\bm{\theta}}) > 0$ as $\bm{y}_\ell - \overline{\bm{\theta}} \in \mathcal{D}$, as desired. 
    \vspace{3 mm} \newline 
    For the other direction, suppose that there exists $\bm{\theta}^{\ast} \in \mathbb{R}^l$ such that $f_0(\bm{y_\ell} - \bm{\theta}^{\ast}) > 0$ for all $1 \le \ell \le n$. Then there exists a satisfying assignment to the original Boolean formula as follows. 
    \vspace{3 mm} \newline 
    Let $\bm{\theta}^{\ast} \in \mathbb{R}^l$ be such that $f_0(\bm{y}_\ell - \bm{\theta}^{\ast}) > 0$ for all $1 \le \ell \le n$, or equivalently, $\bm{y}_\ell - \bm{\theta}^{\ast} \in \mathcal{D}$ for all $1 \le \ell \le n$.
Consider the following assignment of $\mathcal{X}^{\ast}$ of $(a_1, \dots, a_l)$. We will prove that $\mathcal{X}^{\ast}$ is a satisfying assignment: 
   \[ \mathcal{X}^{\ast}(a_j) = \begin{cases} \text{false} &\text{if } \theta_j^{\ast} \in [-\varepsilon, \varepsilon] \\ \text{true} &\text{if } \theta_j^{\ast} \in [-0.5-\varepsilon, -\varepsilon) \end{cases} \]
   We will first prove that this is well-defined. To do this,  we will first show that $\bm{\theta}^{\ast} \in [-0.5-\varepsilon, \varepsilon]^l$. 
   \vspace{3 mm} \newline 
   We recall that for $1 \le \ell \le n$ and $1 \le j \le l$, we have \[ y_{\ell,j} -\theta^{\ast}_j = \begin{cases} -\theta^{\ast}_j &\text{if variable } a_j \ \text{does not occur in } \mathcal{A}_{\ell} \\ \ell -\theta^{\ast}_j &\text{if variable } a_j \ \text{occurs in } \mathcal{A}_{\ell} \end{cases}. \]
   Suppose that variable $a_j$ occurs in some of $\mathcal{A}_{\ell}$, then as $\bm{y}_{\ell} - \bm{\theta}^{\ast} \in \mathcal{D}$ for all $1 \le \ell \le n$ implies that for all $1 \le j \le l$, there exists some $1 \le i \le 7$ such that we have 
   \[ I_{j,\ell,i} \ni y_{\ell,j} - \theta_j^{\ast} = \begin{cases} -\theta^{\ast}_j &\text{if variable } a_j \ \text{does not occur in } \mathcal{A}_{\ell} \\ \ell -\theta^{\ast}_j &\text{if variable } a_j \ \text{occurs in } \mathcal{A}_{\ell} \end{cases} \]
   However, this gives us several constraints on $\theta_j^{\ast}$ based on whether $a_j$ occurs in clause $\mathcal{A}_{\ell}$, and the possible assignment of $a_{j}$ in that clause. In particular, we obtain $l$ constraints on $\theta_j^{\ast}$ in which for all $1 \le \ell \le n$, for some $1 \le i \le 7$, we have  
   \begin{align*}
       \theta_j^{\ast} \in \begin{cases}    [-\varepsilon, \varepsilon] &\text{if } a_j \text{ occurs in } \mathcal{A}_\ell \ \text{and } a_j \ \text{is assigned to be false in } \mathcal{X}_i \\ 
       [-0.5-\varepsilon, -0.5 + \varepsilon] &\text{if } a_j \ \text{occurs in } \mathcal{A}_\ell \ \text{and } a_j \ \text{is assigned to be true in } \mathcal{X}_i\end{cases}.
   \end{align*}
   This is enough to show that $\bm{\theta}^{\ast} \in [-0.5-\varepsilon, \varepsilon]^l$. 
      \vspace{3 mm} \newline 
   To prove that $\mathcal{X}^{\ast}$ is a satisfying assignment for the original Boolean formula, we will show that $\mathcal{X}^{\ast}$ is a satisfying assignment for each clause $\mathcal{A}_1, \ldots, \mathcal{A}_n$. Fix clause $\mathcal{A}_{\ell}$ which contains three variables $a_p, a_q, a_r$. Entries $(p,q,r)$ of $\bm{y}_\ell-\bm{\theta}^\ast$ are $(\ell-\theta^\ast_p, \ell-\theta^\ast_q,\ell-\theta^\ast_r)$, and this 3-vector must lie in $I_{p,\ell,i}\times I_{q,\ell,i}\times I_{r,\ell,i}$ for some $i=1,\ldots,7$.  For example, if $A_\ell=(a_p\vee a_q\vee a_r)$, then for each $i=1,\ldots,7$, at least one of $I_{p,\ell,i}$, $I_{q,\ell,i}$, $I_{r,\ell,i}$ is $[\ell+0.5-\epsilon,\ell+0.5+\epsilon]$. This means that at least one of $\theta^\ast_p,\theta^\ast_q,\theta^\ast_r$ lies in $[-0.5-\epsilon,-0.5+\epsilon]$, which in turn means, by construction of $\mathcal{X}^\ast$, that the setting of $(a_p,a_q,a_r)$ determined by $\bm{\theta}^\ast$ satisfies $A_\ell$.
  % \vspace{3 mm} \newline 
  %We note that by our assignment $\mathcal{X}^{\ast}$,  for some $1 \le i \le 7$, for $j \in \{ p, q, r \}$, we have, by definition of $D_{\ell,i} = I_{1,\ell,i} \times I_{2,\ell,i} \times \dots \times I_{l,\ell,i}$, that as we have
  
 % \[ I_{j,\ell,i} = \begin{cases} [\ell - \varepsilon, \ell + \varepsilon], &\text{if } a_j \ \text{occurs in } A_{\ell} \ \text{and } a_j \ \text{is true in } \mathcal{X}_i, \\ [\ell - \varepsilon, \ell + \varepsilon], &\text{if } a_j \ \text{occurs in } \mathcal{A}_{\ell} \ \text{and } a_j \ \text{is false in } \mathcal{X}_i. \end{cases}  \]
 %  We therefore must have 
 %  \[ \mathcal{X}^{\ast}(a_j) = \begin{cases} \text{true} &\text{if } a_j \ \text{occurs in } \mathcal{A}_i, \  \text{and is assigned to be true in } \mathcal{X}_i \\ \text{false} &\text{if } a_j \ \text{occurs in } \mathcal{A}_i, \  \text{and is assigned to be false in } \mathcal{X}_i \end{cases} \]
 %  in which case we know that $\mathcal{X}^{\ast}$ induced by our three variables $\{ a_p, a_q, a_r \}$ is $\mathcal{X}_i$ for some $1 \le i \le 7$, which by definition, is a satisfying assignment for this clause. This forces $\mathcal{X}^{\ast}$ to be a satisfying assignment for $\mathcal{A}_{\ell}$. As this argument holds for any $1 \le \ell \le n$, this shows that $\mathcal{X}^{\ast}$ is a satisfying assignment for the original Boolean formula. 
 
    \end{proof}
    To finish our proof, we note that by the 3-SAT problem, it is NP-hard to determine whether there exists a satisfying assignment to our original Boolean formula, which by our claim above, implies that it is NP hard to determine whether there exists $\hat{\bm{\theta}} \in \mathbb{R}^l$ such that $f_0(\bm{y}_\ell - \hat{\bm{\theta}}) > 0$ for all $1 \le \ell \le n$, as desired. 
\end{proof} 

We note that for the construction in the preceeding NP-hardness proof, the maximum distance $D$ of any data point or box corner from the origin is $O(\max(n,l)^{1.5})$, the minimum width of any box is $\epsilon=1/80$, and the minimum separation between $\bm{y}_i$’s is 1. Thus, these quantities are polynomially bounded in terms of the size of the input. Therefore the optimal parameters $\bm{\mu}$ and $\sigma$ for the example constructed in the NP-hardness proof could be approximated in polynomial time in the Wasserstein sense using the algorithm of Section \ref{section 2}. 
%{\color{red}
We emphasize, however, that this contrast should be read narrowly. The hardness result in this section concerns global MLE optimization for a finitely specified family in which the parameter affects the support of the density; it is not intended as a blanket statement about the statistical quality of MLE, nor does it rule out the usefulness of local or stationary-point methods in other nonconvex likelihood problems. The intended message is only that, for the class treated here, the global optimization complexity of the MLE and Wasserstein formulations differ sharply. %different.}

\section{Future Work} \label{section 4}
An interesting direction for future research is to develop an analog of the Expectation-Maximization (EM) algorithm using the methods established here. Traditional EM algorithms rely on maximizing the likelihood function for parameter estimation, but this approach can be computationally hard for general distributions, as computing parameters like $\bm{\mu}$ is shown to be NP-hard using MLE. In contrast, incorporating alternative metrics like the Wasserstein distance within the EM framework could provide a more tractable solution, as parameter estimation via minimization under the Wasserstein metric can be achieved in polynomial time. The Wasserstein distance measures similarity between distribution by accounting for the actual geometry between the two distributions through the actual distances between points in the sample space, making it more robust to small sample variations and outliers while preserving the geometric structure of the data.
\vspace{3 mm} \newline 
%{\color{red}
A separate and important direction is statistical analysis. The present paper does not establish consistency rates, asymptotic normality, or finite-sample efficiency comparisons between the Wasserstein estimator and MLE. Understanding whether the computational gain comes with a statistical price, and in which regimes the transport-based estimator is preferable, remains open. This is also where a precise comparison with classical likelihood theory belongs; in the present manuscript, the comparison with MLE is intended only at the level of global optimization complexity.
%}
\vspace{3 mm} \newline 
We also note that our analysis of the hardness result for MLE did not address approximation algorithms, even though we provided a polynomial-time approximation algorithm in the Wasserstein minimization framework. This is because by the standard definition of approximation algorithms, distinguishing between zero and nonzero values is treated equivalently in both exact and approximate contexts. However, alternative definitions of approximation algorithms may exist, under which our hardness result might not hold. Exploring such definitions could offer interesting direction for future research. 
%{\color{red}
Likewise, extending the present analysis beyond the squared Euclidean transport cost, or obtaining sharper statistical guarantees under explicit sampling-noise models, would be natural next steps.
%}
\vspace{3 mm} \newline 
Another potential research question lies in identifying the conditions under which parameter estimation from sample points becomes tractable. We have proven that recovering parameters like $\bm{\mu}$ and $\sigma$ is feasible with Wasserstein distance, though this is not the case for maximum likelihood estimation (MLE). Additionally, in some other work, not detailed here, we can show that parameter recovery in the Wasserstein sense becomes NP-hard for mixture distributions with distinct component means: We can extend Dasgupta's proof of NP-hardness in $k$-clustering problems \cite{%Dasgupta2008TheHO,
Dasgupta2008HardnessKMeans} to prove that finding two optimal means in a mixture distribution is also NP-hard.

\section{Acknowledgements}
We would like to thank Jason Altschuler for help on the early stages of the work. We would also like to thank the referees for their helpful comments.

\section{Statements and Declarations}
There are no conflicts of interest.

There is no data associated with this research.

\bibliographystyle{alpha}
\bibliography{sn-bibliography}
\end{document}